\RequirePackage{tikz}

\documentclass[pdflatex,sn-basic]{sn-jnl}% Basic Springer Nature Reference Style/Chemistry Reference Style

\usepackage{hyperref}
\usepackage{cleveref}
\usepackage{physics}
\usepackage{amsfonts}
\usepackage{amssymb}
\usepackage{amsthm}

\tikzstyle{base}=[rectangle, rounded corners, draw=black,
                           text width =2cm,text centered, font=\sffamily\small]
\tikzstyle{rff}=[base, fill=green!20]
\tikzstyle{qm}=[base, fill=blue!20]
\tikzstyle{normalization}=[base, fill=red!20]

\jyear{2021}%

\theoremstyle{thmstyleone}%
\newtheorem{theorem}{Theorem}%  meant for continuous numbers
\newtheorem{proposition}[theorem]{Proposition}% 

\theoremstyle{thmstyletwo}%

\theoremstyle{thmstylethree}%

\begin{document}

\title[Learning with Density Matrices and Random Features]{Learning with Density Matrices and Random Features}


\author*[1]{\fnm{Fabio A.} \sur{González}}\email{fagonzalezo@unal.edu.co}

\author[1]{\fnm{Alejandro} \sur{Gallego}}\email{jagallegom@unal.edu.co}
\author[1]{\fnm{Santiago} \sur{Toledo-Cortés}}\email{stoledoc@unal.edu.co}
\author[2]{\fnm{Vladimir} \sur{Vargas-Calderón}}\email{vvargasc@unal.edu.co}

\affil*[1]{\orgdiv{MindLab,  Depto. de Ing. de Sistemas e Industrial}, \orgname{Universidad Nacional de Colombia}, \orgaddress{ \city{Bogotá},  \state{DC}, \country{Colombia}}}

\affil[2]{\orgdiv{Grupo de Superconductividad y Nanotecnología, Depto. de Física}, \orgname{Universidad Nacional de Colombia}, \orgaddress{ \city{Bogotá},  \state{DC}, \country{Colombia}}}


\abstract{}
A density  matrix describes the statistical state of a quantum system. It is a powerful formalism to represent both the quantum and classical uncertainty of quantum systems and to express different statistical operations such as measurement, system combination and expectations as linear algebra operations. This paper explores how density matrices can be used as a building block for machine learning models exploiting their ability to straightforwardly combine linear algebra and probability. One of the main results of the paper is to show that density matrices coupled with random Fourier features could approximate arbitrary probability distributions over $\mathbb{R}^n$. Based on this finding the paper builds different models for density estimation, classification and regression. These models are differentiable, so it is possible to integrate them with other differentiable components, such as deep learning architectures and to learn their parameters using gradient-based optimization. In addition, the paper presents optimization-less training strategies based on estimation and model averaging. The models are evaluated in benchmark tasks and the results are reported and discussed.


\maketitle


%\maketitle
\section{Introduction}
\label{sec:introduction}

The formalism of density operators and density matrices was
developed by von Neumann as a foundation of quantum statistical mechanics  \citep{von1927wahrscheinlichkeitstheoretischer}. From the point of view of machine learning, density matrices have an interesting feature: the fact that they combine linear algebra and probability, two of the pillars of machine learning, in a very particular but powerful way.

The main question addressed by this work is how density matrices can be used in machine learning models. One of the main approaches to machine learning is to address the problem of learning as one of estimating a probability distribution from data: joint probabilities $P(x,y)$ in generative supervised models or conditional probabilities  $P(y\vert x)$ in discriminative models. 

The central idea of this work is to use density matrices to represent these probability distributions tackling the important question of how to encode arbitrary probability density functions in $\mathbb{R}^n$ into density matrices. 

The quantum probabilistic formalism of von Neumann is based on linear algebra, in contrast with classical probability which is based on set theory. In the quantum formalism the sample space corresponds to a Hilbert space $\mathcal{H}$ and the event space to a set of linear operators in $\mathcal{H}$, the density operators \citep{wilce2021quantum}.

The quantum formalism generalizes classical probability. A density matrix in an $n$-dimensional Hilbert space can be seen as a catalog of categorical distributions on the finite set $\{1\dots n\}$. A direct application of this fact is not very useful as we want to efficiently model continuous probability distributions in $\mathbb{R}^n$. One of the main results of this paper is to show that it is possible to model arbitrary probability distributions in $\mathbb{R}^n$ using density matrices of finite dimension in conjunction with random Fourier features \citep{rahimi2007rff}. In particular the paper presents a method for non-parametric density estimation that combines density matrices and random Fourier features to efficiently learn a probability density function from data and to efficiently predict the density of new samples.

The fact that the probability density function is represented in matrix form and that the density of a sample is calculated by linear algebra operations makes it easy to  implement the model in GPU-accelerated machine learning frameworks. This also facilitates using density matrices as a building block for classification and regression models, which can be trained using gradient-based optimization and can be easily integrated with conventional deep neural networks. The paper presents examples of these models and shows how they can be trained using gradient-based optimization as well as optimization-less learning based on estimation.

The paper is organized as follows: \Cref{sec:background} covers the background on kernel density estimation, random features,  and density matrices; \Cref{sec:methods} presents four different methods for density estimation, classification and regression; \Cref{sec:related_work} discusses some relevant works; \Cref{sec:experimental_evaluation} presents the experimental evaluation; finally, \Cref{sec:conclusions} discusses the conclusions of the work.

\section{Background and preliminaries}
\label{sec:background}

\subsection{Kernel density estimation}

 Kernel Density Estimation (KDE) \citep{rosenblatt1956, parzen1962estimation}, also known as Parzen-Rossenblat window method, is a non-parametric density estimation method. This method does not make any particular assumption about the underlying probability density function. Given an iid set of samples $X=\{x_1,\dots,x_N\}$, the smooth Parzen's window estimate has the form
 
 \begin{equation}\label{eq:Parzen}
     \hat{f}_\lambda(x) = \frac{1}{NM_{\lambda}}\sum_{i=1}^{N}k_{\lambda}(x,x_i),
 \end{equation}

where $k_{\lambda}(\cdot)$ is a kernel function, $\lambda$ is the smoothing bandwidth parameter of the estimate and $M_{\lambda}$ is a normalizing constant. A small  $\lambda$-parameter implies a small grade of smoothing. 

\citet{rosenblatt1956} and \citet{parzen1962estimation} showed that \cref{eq:Parzen} is an unbiased estimator of the pdf $f$. If $k_{\gamma}$ is the Gaussian kernel, \cref{eq:Parzen} takes the form

 \begin{equation}\label{eq:Parzen-Gaussian}
     \hat{f}_{\gamma}(x) = \frac{1}{NM_\gamma}
     \sum_{i=1}^{N}e^{-\gamma\|x_i-x\|^2},
 \end{equation}
where $M_\gamma = (\pi/\gamma)^{\frac{d}{2}}$.

KDE has several applications: to estimate the underlying probability density function, to estimate confidence intervals and confidence bands \citep{efron1992bootstrap, chernozhukov2014gaussian}, to find local modes for geometric feature estimation \citep{chazal2017robust, chen2016comprehensive}, to estimate ridge of the density function \citep{genovese2014nonparametric}, to build cluster trees \citep{balakrishnan2013cluster}, to estimate the cumulative distribution function \citep{nadaraya1964some}, to estimate receiver operating characteristic (ROC) curves \citep{mcneil1984statistical}, among others.

One of the main drawbacks of KDE is that it is a memory-based method, i.e. it requires the whole training set to do a prediction, which is linear on the training set size. This drawback is typically alleviated by methods that use data structures that support efficient nearest-neighbor queries. This approach still requires to store the whole training dataset.

\subsection{Random features}
\label{subsect:random_features}
Random Fourier features (RFF) \citep{rahimi2007rff} is a method that builds an embedding $\phi_{\mathrm{rff}}:\mathbb{R}^d\to\mathbb{R}^D$ given a shift-invariant kernel $k:\mathbb{R}^d\times\mathbb{R}^d \to \mathbb{R}$  such that $\forall x, y \in \mathbb{R}^d, \ k(x,y) \approx \langle\phi_{\mathrm{rff}}(x),\phi_{\mathrm{rff}}(y)\rangle = \phi_{\mathrm{rff}}^T(x)\phi_{\mathrm{rff}}(y)$. One of the main applications of RFF is to speedup kernel methods, being data independence one of its advantages.

The RFF method is based on the Bochner's theorem. In layman's terms, Bochner's theorem shows that a shift invariant positive-definite kernel $k(\cdot)$ is the Fourier transform of a probability measure $p(w)$. \citet{rahimi2007rff} use this result to approximate the kernel function by designing a sample procedure that estimates the integral of the Fourier transform. The first step is to draw $D$ iid samples $\{w_1,\dots w_D\}$ from $p$ and $D$ iid samples $\{b_1,\dots b_D\}$ from a uniform distribution in $[0,2\pi]$. Then, define:
\begin{align}
\begin{aligned} \label{eq:phi_rff}
\phi_{\text{rff}}:  \mathbb{R}^d &\to \mathbb{R}^D  \\
       x &\mapsto \sqrt{\frac{2}{D}}(\cos\left(w_1^Tx+b_1), \dots, \cos(w_D^Tx+b_D)\right).
\end{aligned}
\end{align}
 
\citet{rahimi2007rff} showed that the expected value of  $\phi_{\mathrm{rff}}^T(x)\phi_{\mathrm{rff}}(y)$ uniformly converges to $k(x,y)$:

\begin{theorem}{\citep{rahimi2007rff}} \label{thm:rff}
    Let $\mathcal{M}$ be a compact subset of $\mathbb{R}^d$ with a diameter $\text{diam}(\mathcal{M})$. Then for the mapping $\phi_{\mathrm{rff}}$ defined above, we have
    
      \begin{align}
      \mathrm{Pr} \left[
      \sup_{x,y \in \mathcal{M}} \vert\phi_{\mathrm{rff}}^T(x)\phi_{\mathrm{rff}}(y) -k(x,y)\vert
      \ge \epsilon \right] \le \nonumber \\
      2^8\left(\frac{\sigma_p \mathrm{diam}(\mathcal{M})}{\epsilon}\right)^2
      \exp\left(-\frac{D\epsilon^2}{4(d+2)}\right),
      \end{align}
      
      where, $\sigma^2_p$ is the second momentum of the Fourier transform of $k$. In particular, for the Gaussian kernel $\sigma^2_p = 2d\gamma$, where $\gamma$ is the spread parameter (see Eq. \ref{eq:Parzen-Gaussian}).
\end{theorem}

Different approaches to compute random features for kernel approximation have been proposed based on different strategies: Monte Carlo sampling \citep{quocle2013fastfood, lee2016orthogonal}, quasi-Monte-Carlo sampling \citep{avron2016quasi, shen2017random}, and quadrature rules \citep{dao2017gaussian}.

RFF may be used to formulate a non-memory based version of KDE. For the Gaussian kernel we have:

\begin{align} \label{eq:phi_density_estimation}
    \hat{f}_{\gamma}(x) 
    & = \frac{1}{NM_\gamma}
     \sum_{i=1}^{N}k_\gamma(x_i, x) \nonumber\\
    & \approx  \frac{1}{NM_\gamma}
     \sum_{i=1}^{N}\langle \phi_{\mathrm{rff}}(x_i),\phi_{\mathrm{rff}}(x) \rangle \nonumber\\
    & =  \frac{1}{M_\gamma}
     \bigg \langle \frac{1}{N} \sum_{i=1}^{N} \phi_{\mathrm{rff}}(x_i),\phi_{\mathrm{rff}}(x) \bigg \rangle \nonumber\\
    & =  \frac{1}{M_\gamma}
     \langle \Phi_{\mathrm{train}},\phi_{\mathrm{rff}}(x)  \rangle \nonumber\\
        & =  \frac{1}{M_\gamma}
     \Phi_{\mathrm{train}}^T\phi_{\mathrm{rff}}(x)   
\end{align}

$\Phi_{\mathrm{train}}$ in \cref{eq:phi_density_estimation} can be efficiently calculated during training time, since  is just an average of the RFF embeddings of the training samples. The time complexity of prediction, \cref{eq:phi_density_estimation}, is constant on the size of the training dataset. The price of this efficiency improvement is a loss in precision, since we are using an approximation of the Gaussian kernel.

\section{Density estimation with density matrices}\label{subsect:density_estimation_with_DM}

The Gaussian kernel satisfy $\forall x,y \in \mathbb{R}^d,  k_\gamma(x,y) > 0$, however the RFF estimation may be negative. To alleviate this we could estimate the square of the kernel and use the fact that $k_{\gamma}(x,y) = k^2_{\gamma/2}(x,y)$. In this case we have:

\begin{align} \label{eq:f_hat_rho_train}
    \hat{f}_{\gamma}(x) 
    & = \frac{1}{NM_\gamma}
     \sum_{i=1}^{N} k_\gamma(x_i, x) \nonumber\\
    & = \frac{1}{NM_\gamma}
     \sum_{i=1}^{N} k^2_{\gamma/2}(x_i, x) \nonumber\\
    & \approx  \frac{1}{NM_\gamma}
     \sum_{i=1}^{N}\langle \phi_{\mathrm{rff}}(x_i),\phi_{\mathrm{rff}}(x) \rangle ^2 \nonumber\\
    & =  \frac{1}{NM_\gamma}
     \sum_{i=1}^{N}\langle \phi_{\mathrm{rff}}(x),\phi_{\mathrm{rff}}(x_i) \rangle \langle \phi_{\mathrm{rff}}(x_i),\phi_{\mathrm{rff}}(x) \rangle \nonumber\\
    & =  \frac{1}{NM_\gamma}
     \sum_{i=1}^{N} \phi_{\mathrm{rff}}^T(x)\phi_{\mathrm{rff}}(x_i)  \phi_{\mathrm{rff}}^T(x_i)\phi_{\mathrm{rff}}(x)  \nonumber\\
     & =  \frac{1}{M_\gamma}
      \phi_{\mathrm{rff}}^T(x)
     \left( \frac{1}{N}\sum_{i=1}^{N} \phi_{\mathrm{rff}}(x_i)  \phi_{\mathrm{rff}}^T(x_i) \right)
     \phi_{\mathrm{rff}}(x)  \nonumber\\
    & =  \frac{1}{M_\gamma}
      \phi_{\mathrm{rff}}^T(x)
     \rho_{\mathrm{train}}
     \phi_{\mathrm{rff}}(x) =:  \hat{f}_{\rho_{\mathrm{train}}}(x) 
\end{align}

In \cref{eq:f_hat_rho_train} it is important to take into account that the parameters of the RFF embedding, $\phi_{\mathrm{rff}}$, are sampled using a parameter $\gamma/2$ for the Gaussian kernel.

The following proposition shows that $\hat{f}_{\rho_{\mathrm{train}}}$, as defined in \cref{eq:f_hat_rho_train}, uniformly converges to the Gaussian kernel Parzen's estimator $\hat{f}_{\gamma}$ (\cref{eq:Parzen-Gaussian}).

\begin{proposition}
\label{prop:DMKDE-approximation}
    Let $\mathcal{M}$ be a compact subset of $\mathbb{R}^d$ with a diameter $\text{diam}(\mathcal{M})$, let $X=\{x_i\}_{i=1\dots N}\subset \mathcal{M}$ a set of iid samples,  then  $\hat{f}_{\rho_{\mathrm{train}}}$ (\cref{eq:f_hat_rho_train}) and $\hat{f}_{\gamma}$  satisfy:
    
      \begin{align}
      \mathrm{Pr} &\left[
      \sup_{x \in \mathcal{M}} 
      \vert\hat{f}_{\rho_{\mathrm{train}}}(x) - \hat{f}_{\gamma}(x)\vert 
      \ge \epsilon 
      \right] \le \nonumber\\
      &2^8\left(\frac{\sqrt{2d\gamma} \mathrm{diam}(\mathcal{M})}{3M_{\gamma}\epsilon}\right)^2
      \exp\left(-\frac{D(3M_{\gamma}\epsilon)^2}{4(d+2)}\right)
      \end{align}
\end{proposition}
\begin{proof}
  (see Apendix \ref{ap:proofs})
\end{proof}

The Parzen's estimator is an unbiased estimator of the true density function from which the samples were generated and Proposition \ref{prop:DMKDE-approximation} shows that  $\hat{f}_{\rho_{\mathrm{train}}}(x)$ can approximate this estimator. 

A further improvement to the $\hat{f}_{\rho_{\mathrm{train}}}(x)$  estimator is to normalize the RFF embedding as follows: 

\begin{align}\label{eq:normalized_rff}
    \ket{\bar{\phi}_{\mathrm{rff}}(x)} = \frac{\phi_{\mathrm{rff}}(x)}{\Vert\phi_{\mathrm{rff}}(x)\Vert}
\end{align}

Here we use the Dirac notation to emphasize the fact that $\bar{\phi}_{\mathrm{rff}}$ is a quantum feature map. This has the effect that the estimation $k_\gamma(x,x) = \braket{\bar{\phi}_{\mathrm{rff}}(x)}  {\bar{\phi}_{\mathrm{rff}}(x)}=1$ will be exact and $\forall x,y \in \mathbb{R}^d,  \braket{\bar{\phi}_{\mathrm{rff}}(x)} {\bar{\phi}_{\mathrm{rff}}(y)} \le 1$. 

During the training phase $\rho_{\mathrm{train}}$ is estimated as the average of the cross product of the normalized RFF embeddings of the training samples:

\begin{align}\label{eq:rho_train}
    \rho_{\mathrm{train}} = \frac{1}{N}\sum_{i=1}^{N} \ket{\bar{\phi}_{\mathrm{rff}}(x_i)} \bra{\bar{\phi}_{\mathrm{rff}}(x_i)}
\end{align}

The time complexity of calculating $\rho_{\mathrm{train}}$ is $O(D^2N)$, i.e. linear on the size of the training dataset. One advantage over conventional KDE is that we do not need to store the whole training dataset, but a more compact representation.

During the prediction phase the density of a new sample is calculated as:

\begin{align}\label{eq:f_hat_rho_train_normal}
    \hat{f}_{\rho_{\mathrm{train}}}(x) = \frac{1}{M_\gamma}
     \bra{\bar{\phi}_{\mathrm{rff}}(x)}
     \rho_{\mathrm{train}}
     \ket{\bar{\phi}_{\mathrm{rff}}(x)}
\end{align}

The $\hat{f}_{\rho_{\mathrm{train}}}$  estimator has an important advantage over the Parzen's estimator, its computational complexity. The time to calculate the Parzen's estimator (\cref{eq:Parzen-Gaussian}) is $O(dN)$ while the time to estimate the density based on the density matrix $\rho_{\mathrm{train}}$ (\cref{eq:f_hat_rho_train_normal}) is $O(D^2)$, which is constant on the size of the training dataset.

The $\rho_{\mathrm{train}}$ matrix in \cref{eq:rho_train} is a well known mathematical object in quantum mechanics, a density matrix, and  \cref{eq:f_hat_rho_train_normal} is an instance of the Born rule which calculates the probability that a measurement of a quantum system produces a particular result. This connection and the basic ideas behind density matrices are discussed in the next section. 

\section{Density matrices}\label{subsect:density_matrices}

This section introduces some basic mathematical concepts that are part of the mathematical framework that supports quantum mechanics and discusses their connection with the ideas introduced in the previous subsection. The contents of this section are not necessary for understanding the rest of the paper and are included to better explain the connection of the ideas presented in this paper with the quantum mechanics mathematical framework.

The state of a quantum system is represented by a vector $\psi \in \mathcal{H}$, where $\mathcal{H}$ is the Hilbert space of the possible states of the system. Usually\footnote{In this paper we use $\mathcal{H} = \mathbb{R}^d$, but most of the methods and results can be extended to the complex case.} $\mathcal{H} = \mathbb{C}^d$. 

As an example, consider a system that could be in two possible states, e.g. the spin of an electron that could be up ($\uparrow$) or down ($\downarrow$) with respect to some axis $z$. The state of this system is, in general, represented by a regular column vector $\ket{\psi} =(\alpha, \beta)$, with $\vert\alpha\vert^2 + \vert\beta\vert^2 = 1$. This state represents a system that is in a superposition of the two basis states $\ket{\psi} = \alpha \uparrow + \beta \downarrow$. The outcome of a measurement of this system, along the $z$ axis, is determined by the Born rule: the spin is up with probability $\vert\alpha\vert^2$ and down with probability $\vert\beta\vert^2$. Notice that $\alpha$ and $\beta$ could be negative or complex numbers, but the Born rule guarantees that we get valid probabilities.

The normalized RFF mapping (\cref{eq:normalized_rff}) can be seen as a function that maps a sample to the state of a quantum system. In quantum machine learning literature, there are different approaches to encode data in quantum states \citep{schuld2018supervised}. The use of RFF as a data quantum encoding strategy was first proposed by \citep{Gonzalez2020, gonzalez2021classification}.

The probabilities that arise from the superposition of states in the previous example is a manifestation of the uncertainty that is inherent to the nature of quantum physical systems. We call this kind of uncertainty \emph{quantum uncertainty}. Other kind of uncertainty comes, for instance, from errors in the measurement or state-preparation processes, we call this uncertainty \emph{classical uncertainty}. A density matrix is a formalism that allows us to represent both types of uncertainty. To illustrate it, let's go back to our previous example. The density matrix representing the state $\psi$ is:

\begin{equation}
    \rho = \ket{\psi} \bra{\psi} = 
    \begin{bmatrix}
        \vert\alpha\vert^2 & \alpha\beta^* \\
        \beta\alpha^* & \vert\beta\vert^2
    \end{bmatrix},
\end{equation}
As a concrete example, consider $\bra{\psi_1} = \left( \frac{1}{\sqrt{2}}, -\frac{1}{\sqrt{2}}\right)$ the corresponding density matrix is:

\begin{equation}
    \rho_1 = \ket{\psi_1} \bra{\psi_1} = 
    \begin{bmatrix}
        \frac{1}{2} & -\frac{1}{2} \\
        -\frac{1}{2} & \frac{1}{2}
    \end{bmatrix},
\end{equation}
which represents a superposition state where we have a $\frac{1}{2}$ probability of measuring any of the two states. Notice that the probabilities for each state are in the diagonal of the density matrix. $\rho_1$ is a rank-1 density matrix, and this means that it represents a pure state. A mixed state, i.e. a state with classical uncertainty, is represented by a density matrix with the form:

\begin{equation}
    \rho = \sum_{i=1}^N p_i \ket{\psi_i} \bra{\psi_i},
\end{equation}
where $p_i > 0 \in\mathbb{R}$, $\sum_{i=1}^N p_i=1$, and  $\{\psi_i\}_{i=1\dots N}$ are the states of a an ensemble of $N$ quantum systems, where each system has an associated probability $p_i$. The matrix $\rho_{\mathrm{train}}$ in \cref{eq:rho_train} is in fact a density matrix that represents the state of an ensemble of quantum systems where each system corresponds to a training data sample. The probability is the same for all the $N$ elements of the ensemble, $\frac{1}{N}$.

As a concrete example of a mixed state consider two pure states $\psi_2 =(1,0)$ and $\psi'_2 =(0,1)$, and consider a system that is prepared in state $\psi_2$ with probability $\frac{1}{2}$ and in state $\psi'_2$ with probability $\frac{1}{2}$ as well. The state of this system is represented by the following density matrix:
\begin{equation}
    \rho_2 = \frac{1}{2}\ket{\psi_2} \bra{\psi_2} + \frac{1}{2}\ket{\psi'_2} \bra{\psi_2^{\prime}}=
    \begin{bmatrix}
        \frac{1}{2} & 0 \\
        0 & \frac{1}{2}
    \end{bmatrix},
\end{equation}

At first sight, states $\rho_1$ and $\rho_2$ may be seen as representing the same quantum system, one where the probability of measuring an up state (or down state) in the $z$ axis is $\frac{1}{2}$. However they are different systems, $\rho_1$ represents a system with only quantum uncertainty, while $\rho_2$ corresponds a system with classical uncertainty. To better observe the differences of the two systems we have to perform a measurement along a particular axis. To do so, we use the following version of the Born rule for density matrices:
\begin{equation}
    P(\varphi\vert\rho) = \text{Tr}(\rho\ket{\varphi}\bra{\varphi})= \bra{\varphi}\rho\ket{\varphi}
\end{equation}
which calculates the probability of measuring the state $\varphi$ in a system in state $\rho$. If we set $\varphi=\left( \frac{1}{\sqrt{2}}, -\frac{1}{\sqrt{2}}\right)$ we get $P(\varphi\vert\rho_1)=1$ and $P(\varphi\vert\rho_2)=\frac{1}{2}$, showing that in fact both systems are different.

% \textcolor{red}{Validmir, please summarize the state of the art in one paragraph if possible. And add one more paragraph emphasizing what is new in our work.}

% The fact that we measure the spin to be up with probability $|\alpha|^2$ is independent from the precision of the measurement. Nonetheless, in the real world we do not have infinitely precise measurement apparatuses, nor can we prepare physical systems to be at Thus, we measure the spin up with a probability close% Before introducing this concept, let us discuss the usual or classical uncertainty. In the classical world, when we perform a measurement we have a degree of uncertainty due to the precision of our measurement apparatuses, calibration issues, among others [cite book Fajardo]. In other words, we assume that there is a hidden variable, a ground truth that we can measure up to some certainty degree, allowing us to estimate the value of such hidden variable.

% On the other hand, multiple experiments at the end of the XIX century and beginnings of the XX century [cite a QM book] showed that classical theory was unable to predict their outcomes. It seemed that there was a probabilistic nature inherent to microscopic systems involved in these experiments. Therefore, the outcome of ...

\section{Methods}
\label{sec:methods}
\subsection{Density matrix kernel density estimation (DMKDE)}
\label{subsec:DMKDE}

In this subsection we present a model for non-parametric density estimation based on the ideas discussed in subsection \ref{subsect:density_estimation_with_DM}. The model receives an input $x \in \mathbb{R}^d$, represents it using a RFF quantum feature map (\cref{eq:phi_rff})  and estimates the density of it using \cref{eq:f_hat_rho_train_normal}. The model can be trained by averaging the density matrices corresponding to the training samples or by using stochastic gradient descent. The second approach requires a re-parametrization of the model that we discuss next.   

The main parameter of the model is $\rho_{\mathrm{train}}$, which is a Hermitian matrix. To ensure this property, we can represent it using a factorization as follows:

\begin{equation}\label{eq:rho_train_factorization}
    \rho_{\mathrm{train}} = V^{T} \Lambda V,
\end{equation}
where $V \in \mathbb{R}^{r \times D}$, $\Lambda \in \mathbb{R}^{r \times r}$ is a diagonal matrix and $r<D$ is the reduced rank of the factorization. With this new representation, \cref{eq:f_hat_rho_train_normal} can be re-expressed as:

\begin{equation}\label{eq:f_hat_rho_factor}
    \hat{f}_{\rho_{\mathrm{train}}}(x)=\frac{1}{M_\gamma}
    \|\Lambda^{\frac{1}{2}}V\bar{\phi}_{\mathrm{rff}}(x)\|^2.
\end{equation}

This reduces the time to calculate the density of a new sample to $O(Dr)$.

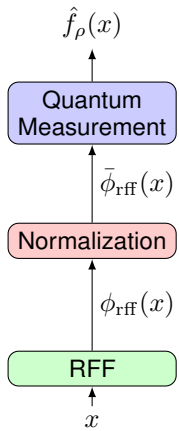
\begin{figure}[tbh]
    \centering
        \begin{tikzpicture}[>=latex, node distance=1.7cm]
        % Nodes
            \node (x) {$x$};
            \node (RFF) [rff, above of=x, yshift=-1cm] {RFF};
            \node (normal) [normalization, above of=RFF] 
            {Normalization};
            \node (QM) [qm, above of= normal] {Quantum Measurement};
            \node (out) [above of= QM, yshift= -.5cm] {$\hat{f}_\rho(x)$};
        % arcs
            \draw[->] (x) -- (RFF);
            \draw[->] (RFF) -- node [xshift=0.6cm] {$\phi_\mathrm{rff}(x)$} (normal);
            \draw[->] (normal) -- node [xshift=0.6cm] {$\bar{\phi}_\mathrm{rff}(x)$} (QM);
            \draw[->] (QM) -- (out);
        \end{tikzpicture}
    \caption{Density matrix kernel density estimation (DMKDE).} \label{fig:DMKDE}
\end{figure}

The model is depicted in Fig. \ref{fig:DMKDE} and its function is governed by the following equations:

\begin{subequations} \label{eq:DMKDE}
\begin{align} 
    z  &   := \phi_{\text{rff}}(x)=\cos(W_{\mathrm{rff}}x+b_{\mathrm{rff}}) \label{eq:DMKDE-1},\\
    z' &   := \frac{z}{\|z\|},\label{eq:DMKDE-2} \\
    \tilde{y} &  := \frac{1}{M_\gamma}\|\Lambda^{\frac{1}{2}}Vz'\|^2 \label{eq:DMKDE-3}
\end{align}
\end{subequations}

The hyperparameters of the model are the dimension of the RFF representation $D$, the spread parameter $\gamma$ of the Gaussian kernel and the rank $r$ of the density matrix factorization. The parameters are the weights and biases of the RFF, $W_{\mathrm{rff}}\in \mathbb{R}^{D\times d}$ and $b_{\mathrm{rff}}\in \mathbb{R}^d$ (corresponding to the $w_i$ and $b_i$ parameters in Eq. \ref{eq:phi_rff}), and the components of the factorization, $V\in \mathbb{R}^{r \times D}$ and $\lambda\in \mathbb{R}^r$, the vector with the elements in the diagonal of $\Lambda$.

The training process of the model is as follows:

\begin{enumerate}
    \item Input. A sample set $X=\{ x_i\}_{i=1\dots N}$ with $x_i \in \mathbb{R}^d$, parameters $\gamma \in \mathbb{R}^+$, $D \in \mathbb{N}$
    \item Calculate $W_{\mathrm{rff}} = [w_1 \dots w_D]$ and $b_{\mathrm{rff}} = [b_1 \dots b_D]$ using the random Fourier features method described in \Cref{subsect:random_features} for approximating a Gaussian kernel with parameters $\gamma/2$ and $D$.
    \item Apply $\bar{\phi}_{\text{rff}}$ (\cref{eq:normalized_rff}):  
    \begin{equation}\label{eq:z_i}
        z_i = \bar{\phi}_{\text{rff}}(x_i).
    \end{equation}
    \item Estimate $\rho_{\mathrm{train}}$:
        \begin{equation}\label{eq:rho_density_estimation}
            \rho_{\mathrm{train}} = \frac{1}{N}\sum_{i=1}^N z_i z_i^T,
        \end{equation}
    \item Make a spectral decomposition of rank $r$ of $\rho_{\mathrm{train}}$:
        \begin{equation*}
            \rho_{\mathrm{train}} = V^T\Lambda V.
        \end{equation*}
\end{enumerate}

Notice that this training procedure does not require any kind of iterative optimization. The training samples are only used once and the time complexity of the algorithm is linear on the number of training samples. The complexity of step 4 is $O(D^2N)$ and of step 5 is $O(D^3)$.

Since the operations defined in \cref{eq:DMKDE} are differentiable, it is possible to use gradient-descent to minimize an appropriate loss function. For instance, we can minimize the negative log-likelihood:
\begin{equation}
    L = -\sum_{i=1}^K \log(\tilde{y})
\end{equation}
In contrast with the learning procedure based on density matrix estimation, using SGD does not guarantee that we will approximate the real density function. If we train all the parameters, maximizing the likelihood becomes an ill-posed problem because of singularities (a Gaussian with arbitrary small variance centered in one training point) \citep{bishop2006pattern}. Keeping fixed the RFF parameters and optimizing the parameters of the density matrix, $V$ and $\lambda$ has shown a good experimental performance. The version of the model trained with gradient descent is called DMKDE-SGD.  

Something interesting to notice is that the process described by \cref{eq:z_i,eq:rho_density_estimation} generalizes density estimation for variables with a categorical distribution, i.e. $x\in\{1,\dots,K\}$. To see this, we replace $\bar{\phi}_{\mathrm{rff}}$ in \cref{eq:z_i} by the well-known one-hot-encoding feature map:

\begin{align} \label{eq:phi_ohe}
\begin{aligned}
\phi_{\text{ohe}}:  D &\to \mathbb{R}^K\\
       i &\mapsto E_i,
\end{aligned}
\end{align}
where $E_i$ is the unit vector with a 1 in position $i$ and 0 in the other positions. It is not difficult to see that in this case
\begin{equation}
    \rho_{ii} = \text{Pr}(x=i) = \frac{\vert\{x_j\vert j \in \{1,\dots,N\}, x_j=i\}\vert}{N}.
\end{equation}

\subsection{Density matrix kernel density classification (DMKDC)}
\label{subsec:DMKDC}
The extension of kernel density estimation to classification is called kernel density classification \citep{hastie2009elements}. The posterior probability is calculated as
\begin{equation}
    \hat{\mathrm{Pr}}(Y=j\vert X=x)=\frac{\pi_j \hat{f}_j(x)}
    {\sum_{k=1}^K\pi_k \hat{f}_k(x)},
\end{equation}
where $\pi_j$ and $\hat{f}_j$ are respectively the class prior and the density estimator of class $j$.

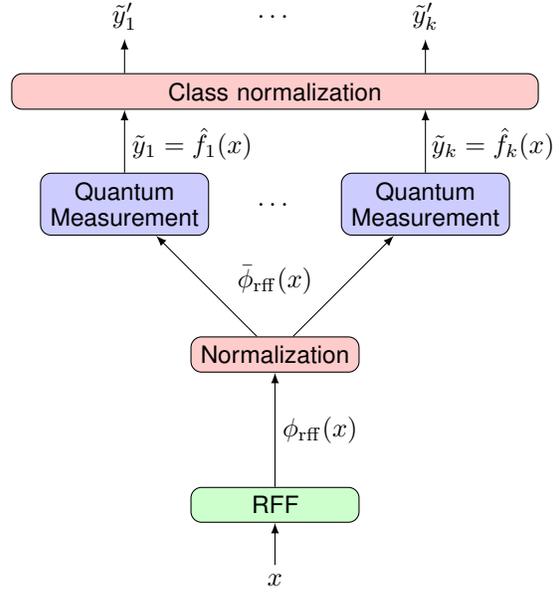
\begin{figure}[tbh]
    \centering
            \begin{tikzpicture}[>=latex,node distance=2 cm]
        % Nodes
            \node (QM_1) [qm] {Quantum Measurement};
            \node (out_qm_1) [above of= QM_1, yshift= -1.2cm, xshift= 0.9cm]  {$\tilde{y}_1 = \hat{f}_1(x)$};
            \node (dots_1) [right of=QM_1] {$\dots$};
            \node (dots_2) [above of=dots_1, yshift=.5cm] {$\dots$};
            \node (QM_k) [qm, right of=dots_1] {Quantum Measurement};
            \node (out_qm_k) [above of= QM_k, yshift= -1.2cm, xshift= 0.9cm]  {$\tilde{y}_k = \hat{f}_k(x)$};
            \node (normal) [normalization, below of=dots_1] {Normalization};
            \node (normal_phi) [above of= normal, yshift= -1cm] {$\bar{\phi}_\mathrm{rff}(x)$} (QM);
            \node (RFF) [rff, below of=normal] {RFF};
            \node (x) [below of=RFF, yshift=1cm] {$x$};
            \node (class_normal) [normalization, 
                                  minimum width=7cm, 
                                  text width=4cm, 
                                  above of=dots_1,
                                  yshift=-0.5cm] {Class normalization};
            \node (out_1) [above of=QM_1, yshift=.5cm] {$\tilde{y}'_1$};
            \node (out_k) [above of=QM_k, yshift=.5cm] {$\tilde{y}'_k$};
        % arcs
            \draw[->] (x) -- (RFF);
            \draw[->] (RFF) -- node [xshift=0.6cm] {$\phi_\mathrm{rff}(x)$} (normal);
            \draw[->] (normal) -- (QM_1);
            \draw[->] (normal) -- (QM_k);
            \draw[->] (QM_1) -- (QM_1 |- class_normal.south);
            \draw[->] (QM_k) -- (QM_k |- class_normal.south);
            \draw[->] (QM_1 |- class_normal.north) -- (out_1);
            \draw[->] (QM_k |- class_normal.north) -- (out_k);
        
        \end{tikzpicture}
    \caption{Density matrix kernel density classification (DMKDC).} \label{fig:DMKDC}
\end{figure}

We follow this approach to define a classification model that uses the density estimation strategy based on RFF and density matrices described in the previous section. The input to the model is a vector $x \in \mathbb{R}^d$. The model is depicted in Fig. \ref{fig:DMKDC} and defined by the following equations:
\begin{subequations}
\begin{align}
    z  &   := \phi_{\text{rff}}(x)=\cos(W_{\mathrm{rff}}x+b_{\mathrm{rff}}) \label{eq:DMKDC-1},\\
    z' &   := \frac{z}{\|z\|},\label{eq:DMKDC-3} \\
    \tilde{y}_i &  := \|\Lambda_i^{\frac{1}{2}}V_iz'\|^2 \ \ \forall i=1\dots K,\label{eq:DMKDC-4} \\
    \tilde{y}'_i & := \frac{\pi_i \tilde{y}_i}{\sum_{j=i}^K \tilde{y}_j} \ \ \forall i=1\dots K, \label{eq:DMKDC-2}
\end{align}
\end{subequations}

The hyperparameters of the model are the dimension of the RFF representation $D$, the spread parameter $\gamma$ of the Gaussian kernel, the class priors $\pi_i$ and the rank $r$ of the density matrix factorization. The parameters are the weights and biases of the RFF, $W_{\mathrm{rff}}\in \mathbb{R}^{D\times d}$ and $b_{\mathrm{rff}}\in \mathbb{R}^d$, and the components of the factorization, $V_i\in \mathbb{R}^{r \times D}$ and $\lambda_i\in \mathbb{R}$ for $i=1\dots K$.

The model can be trained using two different strategies: one, using DMKDE to estimate the density matrices of each class; two, use stochastic gradient descent over the parameters to minimize an appropriate loss function. 

The training process based on density matrix estimation is as follows:
\begin{enumerate}
    \item Use the RFF method to calculate $W_{\mathrm{rff}}$ and $b_{\mathrm{rff}}$.
    \item For each class $i$:
    \begin{enumerate}
        \item Estimate $\pi_i$ as the relative frequency of the class $i$ in the dataset.
        \item Estimate $\rho_i$ using \cref{eq:rho_density_estimation} and the training samples from class $i$.
        \item Find a factorization of rank $r$ of $\rho_i$:
        \begin{equation*}
            \rho_i = V_i^T\Lambda V_i.
        \end{equation*}
    \end{enumerate}
\end{enumerate}

Notice that this training procedure does not require any kind of iterative optimization. The training samples are only used once and the time complexity of the algorithm is linear on the number of training samples. The complexity of step 2(b) is $O(D^2N)$ and of 2(c) is $O(D^3)$.

Since the operations defined in \cref{eq:DMKDC-1,eq:DMKDC-3,eq:DMKDC-4,eq:DMKDC-2} are differentiable, it is possible to use gradient-descent to minimize an appropriate loss function. For instance, we can use categorical cross entropy:
\begin{equation}
    L = \sum_{i=1}^K y_i\log(\tilde{y}'_i)
\end{equation}
where $y=(y_1,\dots,y_K)$ corresponds to the one-hot-encoding of the real label of the sample $x$. The version of the model trained with gradient descent is called DMKDC-SGD. 

An advantage of this approach is that the model can be jointly trained with other differentiable architecture such as a deep learning feature extractor.

\subsection{Quantum measurement classification (QMC)}
\label{subsec:QMC}
In DMKDC we assume a categorical distribution for the output variable. If we want a more general probability distribution we need to define a more general classification model. The idea is to model the joint probability of inputs and outputs using a density matrix. This density matrix represents the state of a bipartite system whose representation space is $\mathcal{H}_{\mathcal{X}} \otimes \mathcal{H}_{\mathcal{Y}}$ where $\mathcal{H}_{\mathcal{X}}$ is the representation space of the inputs, $\mathcal{H}_{\mathcal{Y}}$ is the representation space of the outputs and $\otimes$ is the tensor product. A prediction is made by performing a  measurement with an operator specifically prepared from a new input sample. 

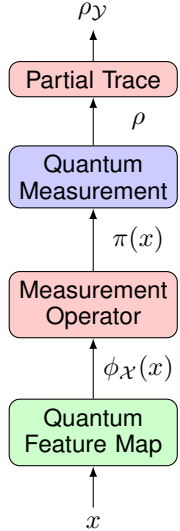
\begin{figure}[tbh]
    \centering
        \begin{tikzpicture}[>=latex, node distance=1.7cm]
        % Nodes
            \node (x) {$x$};
            \node (q_feature_map) [rff, above of=x, yshift=-.5cm] {Quantum Feature Map};
            \node (m_operator) [normalization, above of=q_feature_map] 
            {Measurement Operator};
            \node (QM) [qm, above of= m_operator] {Quantum Measurement};
            \node (partial_trace) [normalization, above of=QM, yshift= -.4cm] 
            {Partial Trace};
            \node (out) [above of= partial_trace, yshift= -.8cm] {$\rho_{{\mathcal{Y}}}$};
        % arcs
            \draw[->] (x) -- (q_feature_map);
            \draw[->] (q_feature_map) -- node [xshift=0.6cm] {$\phi_{\mathcal{X}}(x)$} (m_operator);
            \draw[->] (m_operator) -- node [xshift=0.6cm] {$\pi(x)$} (QM);
            \draw[->] (QM) -- node [xshift=0.6cm] {$\rho$} (partial_trace);
            \draw[->] (partial_trace) -- (out);
        \end{tikzpicture}
    \caption{Quantum measurement classification (QMC).} \label{fig:QMC}
\end{figure}

This model is based on the one described by \citet{Gonzalez2020} and is depicted in \Cref{fig:QMC} and works as follows:

\begin{itemize}
    \item Input encoding. The input $x\in \mathbb{R}^d$ is encoded using a feature map $\phi_{\mathcal{X}}$
    \begin{equation}
        z := \phi_{\mathcal{X}}(x).
    \end{equation}
    \item Measurement operator. The effect of this measurement operator is to collapse, using a projector to $z$, the part $\mathcal{H_X}$  of the bipartite system while keeping the $\mathcal{H_Y}$ part unmodified. This is done by defining the following operator:
    \begin{equation}
        \pi := zz^T \otimes \mathrm{Id}_{\mathcal{H_Y}},
    \end{equation}
    where $\mathrm{Id}_{\mathcal{H_Y}}$ is the identity operator in $\mathcal{H_Y}$.
    \item Apply the measurement operator to the training density matrix:
    \begin{equation}
        \rho := \frac{\pi\rho_{\mathrm{train}}\pi}
    {\Tr[\pi\rho_{\mathrm{train}}\pi]},
    \end{equation}
    \item Calculate the partial trace of $\rho$ with respect to $\mathcal{X}$ to obtain a density matrix that encodes the prediction:
    \begin{equation}
        \rho_{{\mathcal{Y}}} := \Tr_{\mathcal{X}}[\rho].
    \end{equation}
\end{itemize}

The parameter of the model, without taking into account the parameters of the feature maps, is the $\rho_{\mathrm{train}} \in \mathbb{R}^{D_{\mathcal{X}}D_{\mathcal{Y}}\times D_{\mathcal{X}}D_{\mathcal{Y}}}$  density matrix, where $D_{\mathcal{X}}$ and $D_{\mathcal{Y}}$ are the dimensions of $\mathcal{H_X}$ and $\mathcal{H_Y}$ respectively. As discussed in \Cref{subsec:DMKDE}, the density matrix $\rho_{\mathrm{train}}$ can be factorized as:
\begin{equation}\label{eq:rho_train_factorization2}
    \rho_{\text{train}} = V^T \Lambda V^{\ }
\end{equation}
where $V \in \mathbb{R}^{r \times D_{\mathcal{X}}D_{\mathcal{Y}}}$, $\Lambda \in \mathbb{R}^{r \times r}$ is a diagonal matrix and $r<D_{\mathcal{X}}D_{\mathcal{Y}}$ is the reduced rank of the factorization. This factorization not only helps to reduce the space necessary to store the parameters, but learning $V$ and $\Lambda$, instead of $\rho_{\mathrm{train}}$, helps to guarantee that $\rho_{\mathrm{train}}$ is a valid density matrix.

As in Subsection \ref{subsec:DMKDC}, we described two different approaches to train the system: one based on estimation of the $\rho_{\text{train}}$ and one based on learning $\rho_{\text{train}}$ using gradient descent. QMC can be also trained using these two strategies.

In the estimation strategy, given a training data set $\{(x_i, y_i)\}_{i=1\dots N}$ the training density matrix is calculated by:

        \begin{equation}\label{eq:QMC_rho_density_estimation}
            \rho_{\mathrm{train}}=\frac{1}{N}\sum_{i=1}^N  
            \left(\phi_{\mathcal{X}}(x_i)\otimes\phi_{\mathcal{Y}}(y_i)\right)
            \left(\phi_{\mathcal{X}}(x_i)\otimes\phi_{\mathcal{Y}}(y_i)\right)^T.
        \end{equation}

The computational cost is $O(ND^2_{\mathcal{X}}D^2_{\mathcal{Y}}$).

For the gradient-descent-based strategy (QMC-SGD) we can minimize the following loss function:

\begin{equation}
    L = \sum_{i=1}^{D_{\mathcal{Y}}} y_i\log(\rho_{\mathcal{Y}ii}),
\end{equation}
where $\rho_{\mathcal{Y}ii}$ is the $i$-th diagonal element of $\rho_{\mathcal{Y}}$.

As in DMKDC-SGD, this model can be combined with a deep learning architecture and the parameters can be jointly learned using gradient descent.

QMC can be used with different feature maps for inputs and outputs. For instance, if $\phi_{\mathcal{X}}=\phi_{\mathrm{rff}}$ (\cref{eq:phi_rff}) and $\phi_{\mathcal{Y}}=\phi_{\mathrm{ohe}}$ (\cref{eq:phi_ohe}), QMC corresponds to DMKDC. However, in this case DMKDC is preferred because  of its reduced computational cost.

\subsection{Quantum measurement regression (QMR)}
\label{subsec:QMR}
In this section we show how to use QMC to perform regression. For this we will use a feature map that allows us to encode continuous values. 
%This feature map is based on the softmax quantum feature map presented in \citep{Gonzalez2020}. 
The feature map is defined with the help of $D$ equally distributed landmarks in the $[0,1]$ interval\footnote{Without loss of generality the continuous variable to be encoded is restricted to the $[0,1]$ interval.}:  

\begin{equation}
    \alpha_i = \frac{i-1}{D-1} \ \text{for } i=1\dots D.
\end{equation}
The following function (which is equivalent to a softmax) defines a set of unimodal probability density functions centered at each landmark:

\begin{align} 
    p_i(x)=\left(\frac{\exp (-\beta\|x-\alpha_i\|^2)}{\sum_{j=1}^m \exp(-\beta\|x-\alpha_j\|^2)}\right)_{i=1\dots D},
\end{align}
where $\beta$ controls the shape of the density functions. 

The feature map is defined as:

\begin{align} \label{eq:phi_sm}
\begin{aligned}
\phi_{\text{sm}}:  [0,1] &\to \mathbb{R}^D\\
       x &\mapsto (\sqrt{p_1(x)}, \dots, \sqrt{p_D(x)}). 
\end{aligned}
\end{align}

This feature map is used in QMC as the feature map of the output variable ($\phi_{\mathcal{Y}}$). To calculate the prediction for a new sample $x$ we apply the process described in Subsection \ref{subsec:QMC} to obtain $\rho_{\mathcal{Y}}$. Then the prediction is given by:
\begin{equation}
    \hat{y} = E_{\rho_{\mathcal{Y}}}[\alpha_i] = \sum_{i=1}^D \rho_{\mathcal{Y}ii}\alpha_i.
\end{equation}

Note that this framework also allows to easily compute confidence intervals for the prediction. The model can be trained using the strategies discussed in Subsection \ref{subsec:QMC}. For gradient-based optimization we use a mean squared error loss function:

\begin{equation}
    L = \sum_{i=1}^D (y-\hat{y})^2 + \alpha \sum_{i=1}^D \rho_{\mathcal{Y}ii} (\hat{y}-\alpha_i)^2
\end{equation}
where the second term correspond to the variance of the prediction and $\alpha$ controls the trade-off between error and variance.

\section{Related Work}
\label{sec:related_work}

The ability of density matrices to represent probability distributions has been used in previous works. The early work by \citet{lior} uses the density matrix formalism to perform spectral clustering, and shows that this formalism not only is able to predict cluster labels for the objects being classified, but also provides the probability that the object belongs to each of the clusters. Similarly, \citet{tiwariQIBC2019} proposed a quantum-inspired binary classifier using density matrices, where samples are encoded into pure quantum states. In a similar fashion, \citet{Sergioli2018} proposed a quantum nearest mean classifier based on the trace distance between the quantum state of a sample, and a quantum centroid that is a mixed state of the pure quantum states of all samples belonging to a single class. Another class of proposals directly combine these quantum ideas with customary machine learning techniques, such as frameworks for multi-modal learning for sentiment analysis~\citep{LI202158,melucci2020ieee,ZHANG201821}. 

Since its inception, random features have been used to improve the performance of several kernel methods: kernel ridge regression \citep{avron2017random}, support vector machines (SVM) \citep{sun2018but}, and nonlinear component analysis \citep{xie2015scale}. Besides, random features have been used in conjunction with deep learning architectures in different works \citep{arora2019exact, ji2019polylogarithmic, li2019implicit}.

The combination of RFF and density matrices was initially proposed by~\citet{Gonzalez2020}. In that work, RFF are used as a quantum feature map, among others, and the QMC method (Subsection \ref{subsec:QMC}) was presented. In~\citet{Gonzalez2020} the coherent state kernel showed better performance than the Gaussian kernel. It is important to notice that the coherent state kernel was calculated exactly while the Gaussian kernel was approximated using RFF. It is possible to apply RFF to approximate the coherent state kernel and use it as the quantum feature map in the models presented in this paper. The emphasis of~\citet{Gonzalez2020} is to show that quantum measurement can be used to do supervised learning. In contrast, the present paper addresses a wider problem with several new contributions: a new method for density estimation based on density matrices and RFF, the proof of the connection between this method and kernel density estimation, and new differentiable models for density estimation, classification and regression.

The present work can be seen as a type of quantum machine learning (QML), which is generally referred as the field in the intersection of quantum computing and machine learning~\citep{schuld2015introduction,schuld2018supervised}. In particular, the methods in this paper are in the subcategory of QML called quantum inspired classical machine learning, where theory and methods from quantum physics are borrowed and adapted to machine learning methods intended to run in classical computers. Works in this category include:  quantum-inspired recommendation systems~\citep{tang2019quantum}, quantum-inspired kernel-based classification methods~\citep{tiwari2020kernel,Gonzalez2020}, conversational sentiment analysis based on density matrix-like convolutional neural networks~\citep{zhang2019conversational}, dequantised principal component analysis~\citep{tang2019quantuminspired}, among others. 

Being a memory-based strategy, KDE suffers from large-scale, high dimensional data. Due to this issue, fast approximate evaluation of non-parametric density estimation is an active research topic. Different approaches are proposed in the literature: higher-order divide-and-conquer method \citep{gray2003nonparametric}, separation of near and far-field (pruning) \citep{march2015askit}, and hashing based estimators (HBE) \citep{charikar2017hashing}. Even though the purpose of the present work was not to design methods for fast approximation of KDE, the use of RFF to speed KDE seems to be a promising research direction. Comparing DMKDE against fast KDE approximation methods is part of our future work.

\section{Experimental Evaluation}
\label{sec:experimental_evaluation}
In this section we perform some experiments to evaluate the performance of the proposed methods in different benchmark tasks. The experiments are organized in three subsections: density estimation evaluation, classification evaluation and ordinal regression evaluation. The source code of the methods and the scripts of the experiments are available at \url{https://drive.google.com/drive/folders/16pHMLjIvr6v1zY6cMvo11EqMAMqjn3Xa} as Jupyter notebooks. 

\subsection{Density estimation evaluation}

The goal of these experiments is to evaluate the efficacy and efficiency of DMKDE to approximate a pdf. We compare it against conventional Gaussian KDE.

\subsubsection{Data sets and experimental setup}

We used three datasets:

\begin{itemize}
    \item 1-D synthetic. The first synthetic dataset corresponds to a mixture of univariate Gaussians as shown in Figure \ref{fig:probability_kde}. The mixture weights are 0.3 and 0.7 respectively and the parameters are $(\mu_1=0,\sigma=1)$ and $(\mu_1=5,\sigma=1)$. We generated 10,000 samples for training and use as test dataset 1,000 samples equally spaced in the interval $[-5,10]$.
    \item 2-D synthetic. This dataset corresponds to three spirals as depicted in Figure \ref{fig:spirals-dmkde-likelihood}.The training and test datasets have 10,0000 and 1,000 points respectively, all of them generated with the same stochastic procedure.
    \item MNIST dataset. We used PCA to reduce the original 784 dimension to 40. The resulting vectors were scaled to $[0,1]$. We used stratified sampling to choose 10,000 and 1,000 samples for training and testing respectively.
\end{itemize}

\begin{figure}[tbh]
\begin{centering}
\includegraphics[scale=0.6]{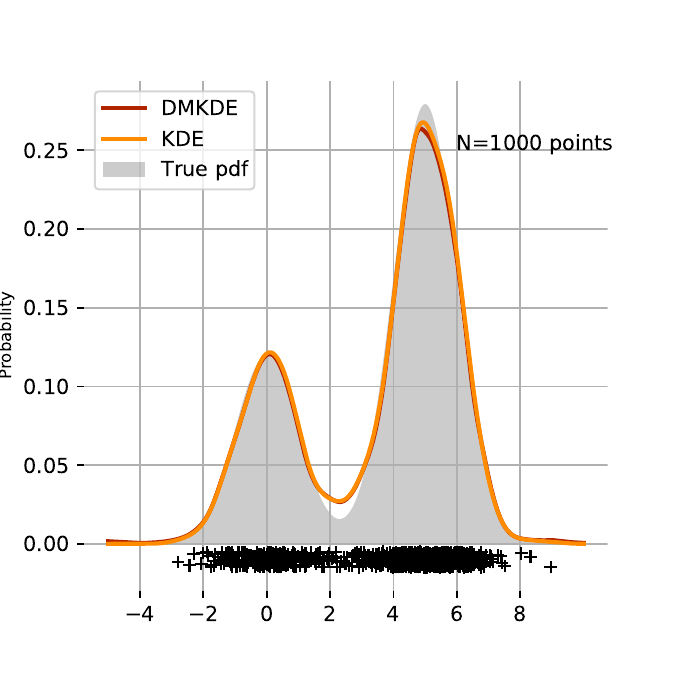}
\par\end{centering}
\caption{1-D synthetic dataset. The gray zone is the area of the true density. The estimated pdf by DMKDE ($\gamma=2$) and KDE ($\gamma=4$) is shown.
\label{fig:probability_kde}}
\end{figure}

We performed two types of experiments over the three datasets. In the first, we wanted to evaluate the accuracy of DMKDE. In the second, we evaluated the time to predict the density on the test set. 

In the first experiment, DMKDE was run with different number of RFF to see how the dimension of the RFF representation affected the accuracy of the estimation. For the 1-D dataset, both the DMKDE prediction and the KDE prediction were compared against the true pdf using root mean squared error (RMSE). For the 2-D dataset the RMSE between the DMKDE prediction and the KDE prediction was evaluated. In the case of MNIST, and because of the small values for the density, we calculated the RMSE between the log density predicted by DMKDE and KDE. The number of eigencomponents ($r$) was chosen by sorting the eigenvalues in descending order and plotting them to look for the curve elbow. For the 1-D and 2-D datasets, the $\gamma$ value was chosen to get a good approximation of the data density, this was visually verified. For the MNIST dataset, the $\gamma$ value was chosen by looking at a histogram of pairwise distances of the data. The value of the parameters were: $(\gamma=16, r=30)$ for the 1-D dataset, $(\gamma=256, r=100)$ for the 2-D dataset, $(\gamma=1, r=150)$ for the MNIST dataset.

For the second experiment, we measured the time taken to predict 1,000 test samples for both KDE and DMKDE using different number of train samples. KDE was implemented in Python using liner algebra operations accelerated by numpy. At least for the experiments reported, our implementation was faster than other KDE implementations available such as the one provided by scikit learn (\url{https://scikit-learn.org/stable/modules/density.html}), which is probably optimized for other use cases. DMKDE was implemented in Python using Tensorflow. The main reason for using Tensorflow was its ability to automatically calculate the gradient of computational graphs. KDE could not benefit from this feature, on the contrary, its performance could be hurt by Tensorflow's larger memory footprint. Another advantage of Tensorflow is its ability to generate code optimized for a GPU, so both methods were run on a 2.20 GHz dual-core Intel(R) Xeon(R) CPU without a GPU to avoid any unfair advantage.

\subsubsection{Results and discussion}

Figure \ref{fig:dmkdevsrawkde} shows how the accuracy of DMKDE increases with an increasing number of RFF. For each configuration 30 experiments were run and the blue solid line represents the mean RMSE of the experiments and the blue region represents the 95\% confidence interval. In all the three datasets, $2^{10}$ RFF achieved a low RMSE. The variance also decreases when the number of RFF is increased. 

\begin{figure*}[tbh]
\begin{centering}
\includegraphics[scale=0.43]{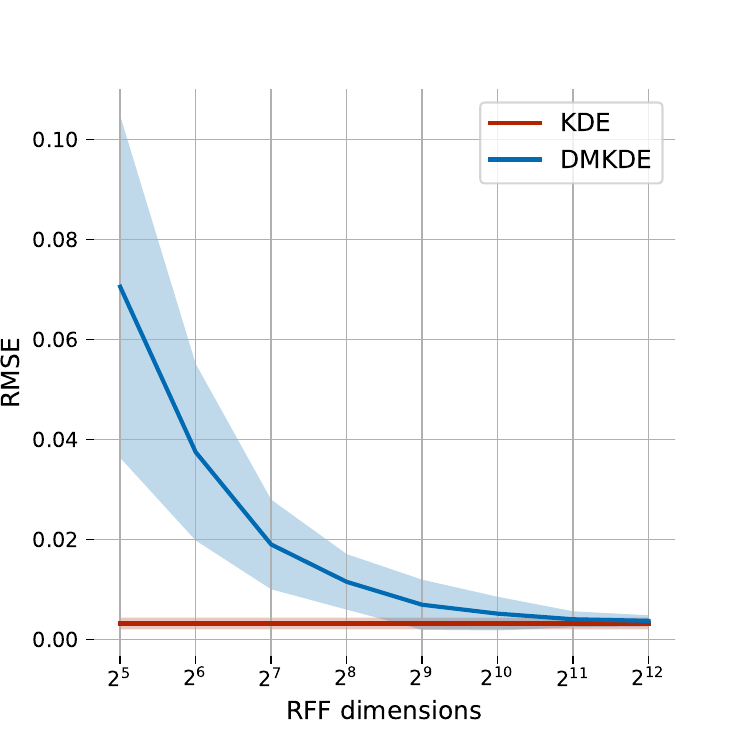}
\includegraphics[scale=0.43]{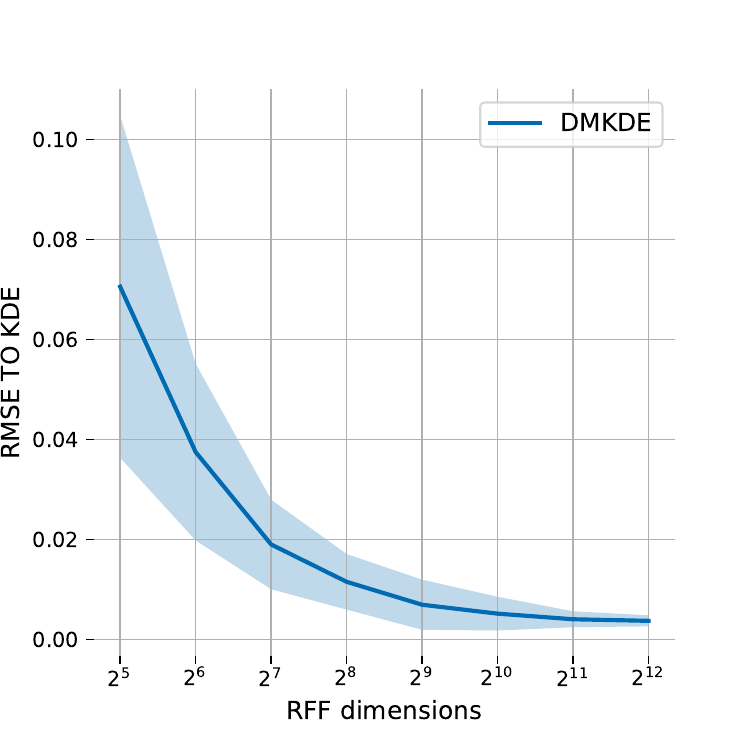}
\includegraphics[scale=0.43]{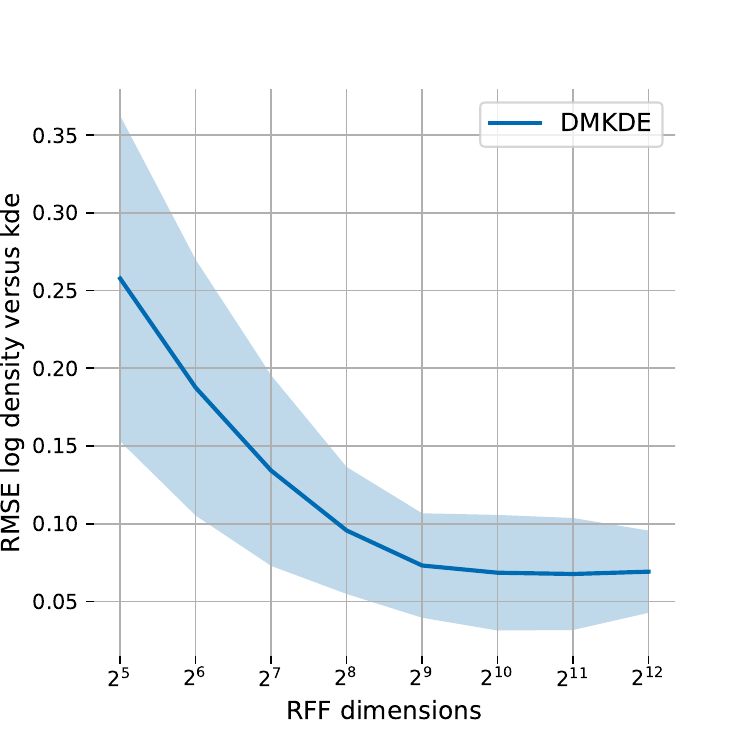}
\par\end{centering}
\caption{Accuracy of the density estimation of DMKDE for different number of RFF for the 1-D dataset (top left), 2-D dataset (top right) and MNIST dataset (bottom). For the 1-D dataset both KDE and DMKDE are compared against the true density. For the two other datasets the difference between KDE and DMKDE is calculated. In all the cases the RMSE is calculated. The blue shaded zone represents the 95\% confidence interval. \label{fig:dmkdevsrawkde}}
\end{figure*}

\begin{figure*}[tbh]
\begin{centering}
\includegraphics[scale=0.33]{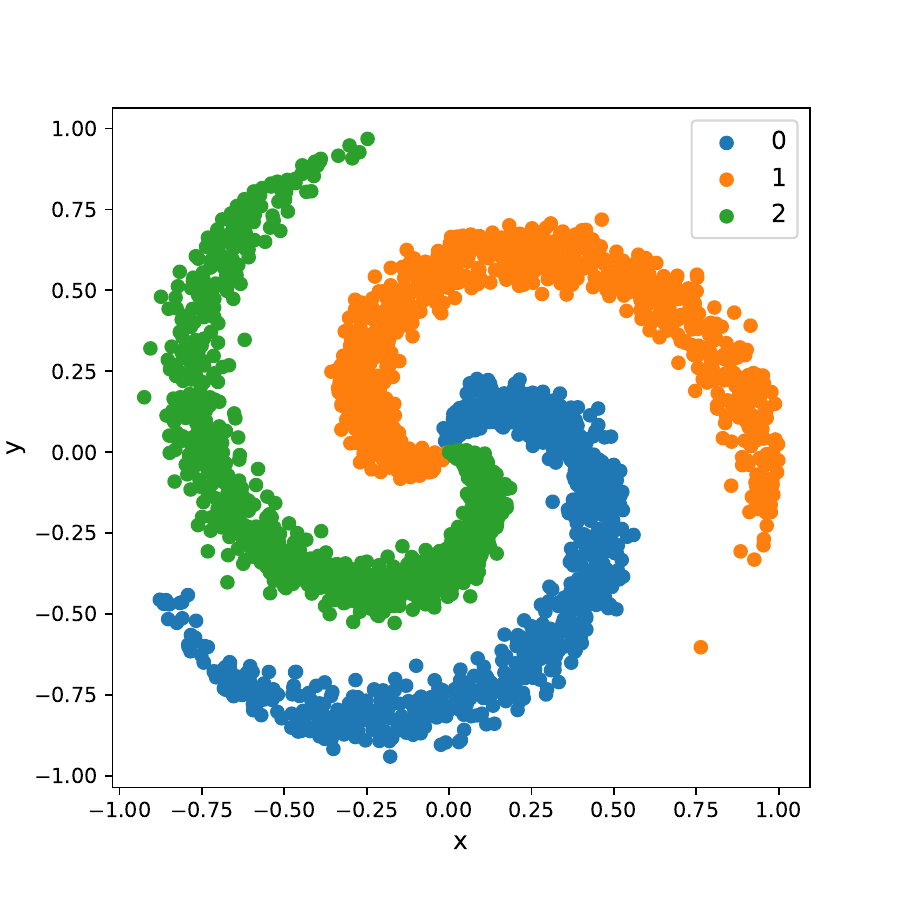}
\includegraphics[scale=0.33]{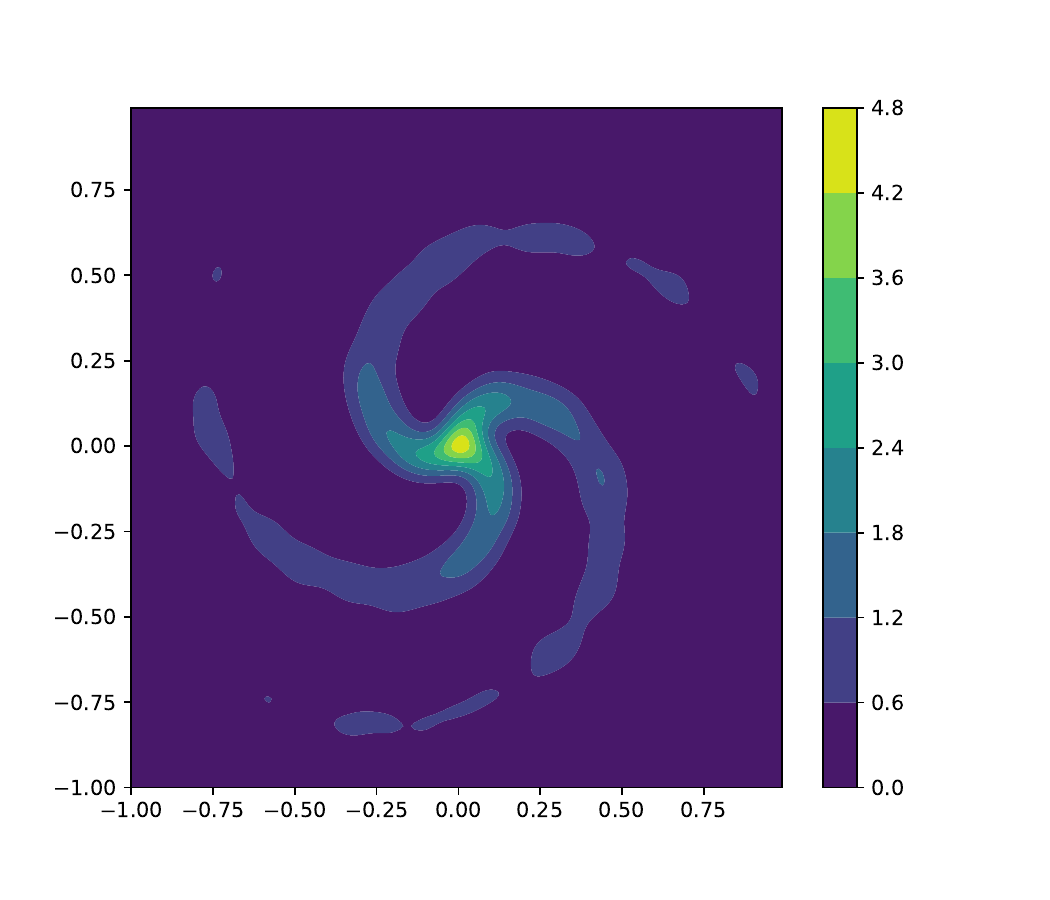}
\includegraphics[scale=0.33]{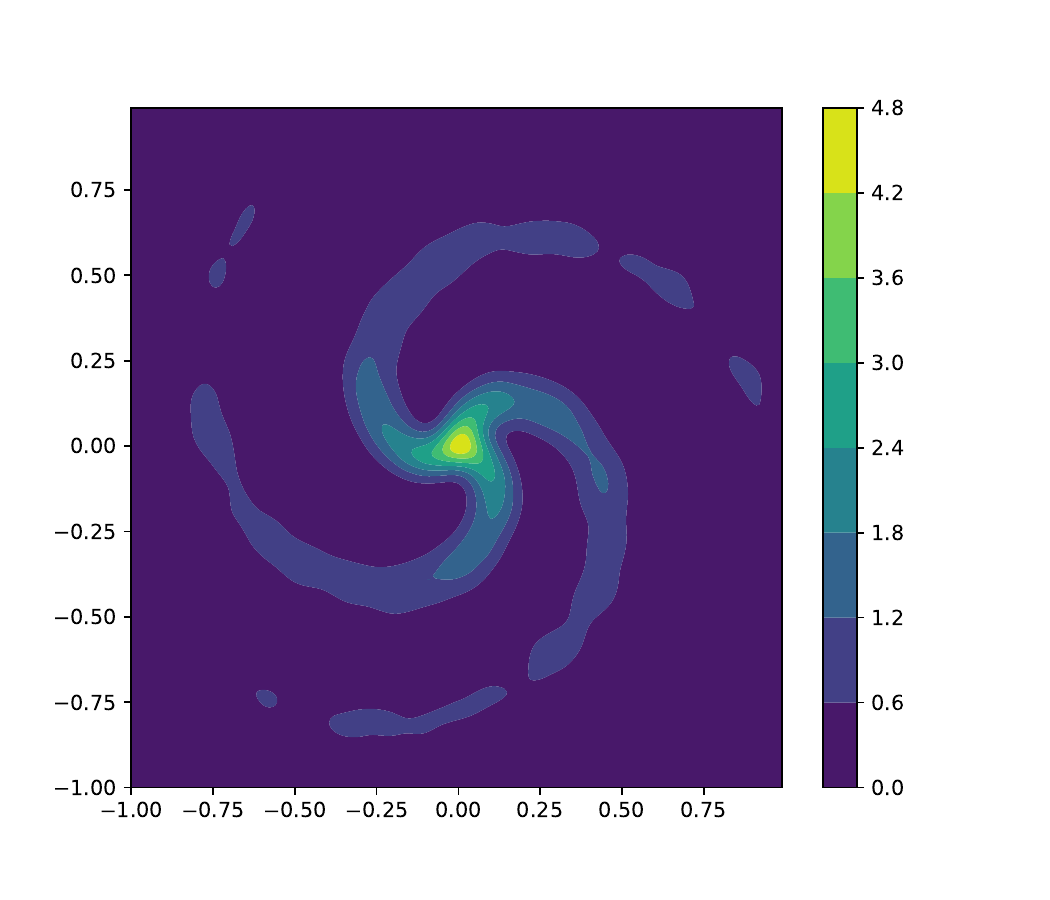}
\par\end{centering}
\caption{2-D spirals dataset (top left) and the density estimation of both KDE (top right) and DMKDE (bottom). \label{fig:spirals-dmkde-likelihood}}
\end{figure*}

\begin{figure*}[tbh]
\begin{centering}
\includegraphics[scale=0.39]{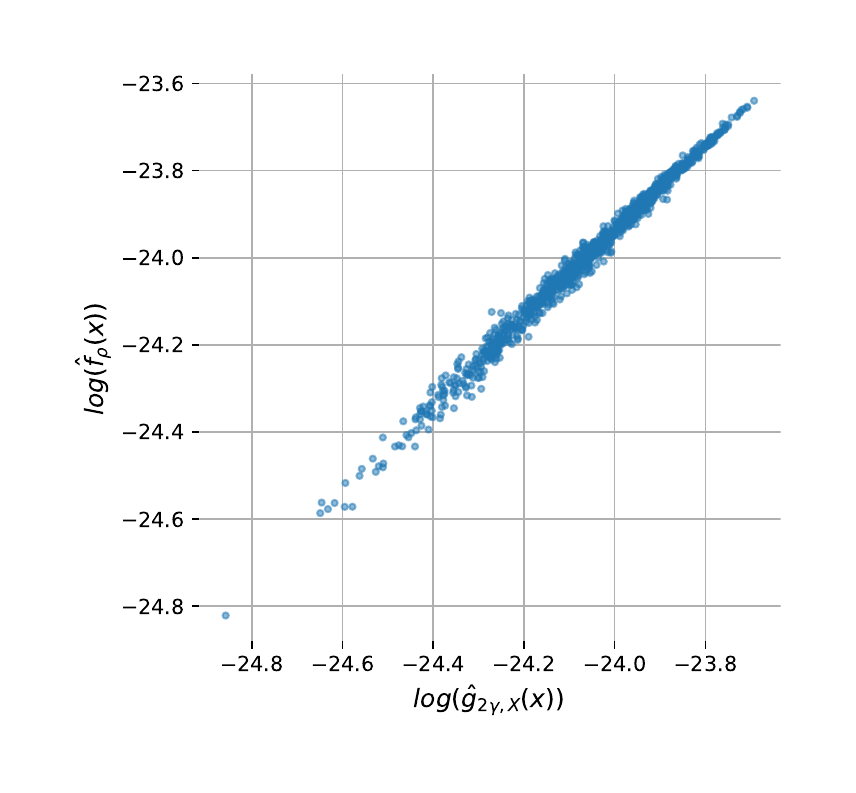}
\includegraphics[scale=0.39]{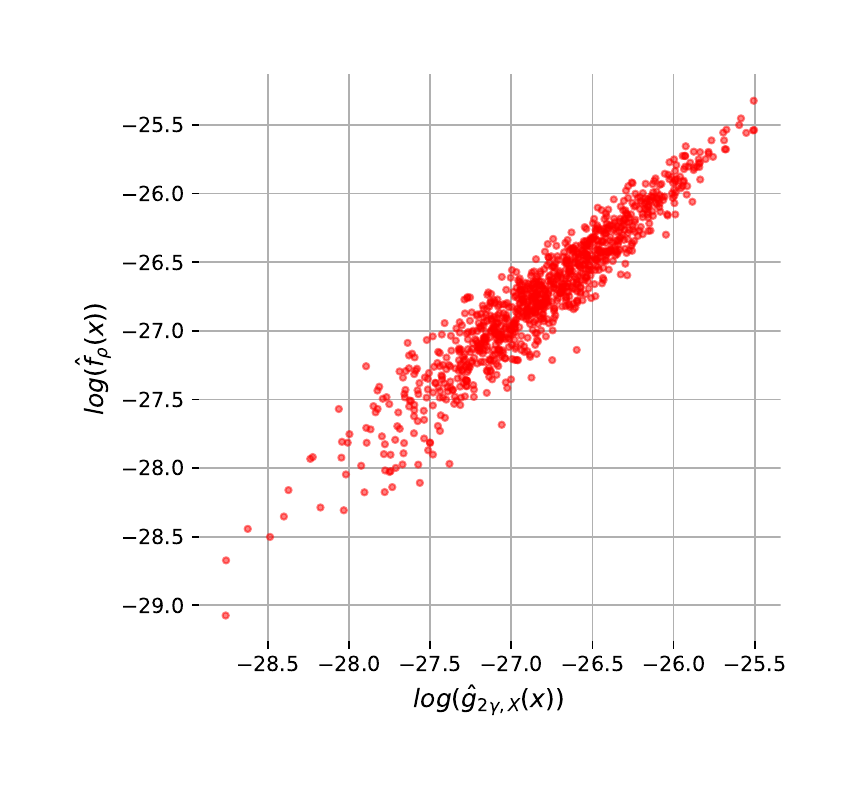}
\includegraphics[scale=0.39]{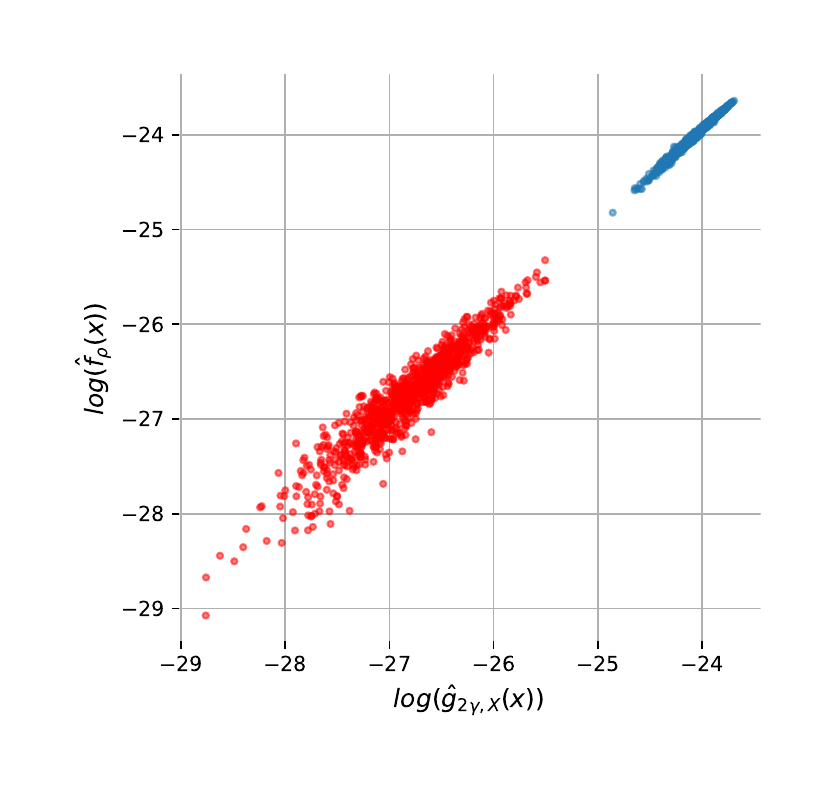}
\par\end{centering}
\caption{Scatter-plots comparing the log density predicted by KDE and DMKDE: test samples (top left), uniformly random generated samples (top right), both test and random samples (bottom). \label{fig:mnist-scatter}}
\end{figure*}

Figure \ref{fig:spirals-dmkde-likelihood} shows the 2-D spirals dataset (left) and the density estimation of both KDE (center) and DMKDE (right).  The density calculated by DMKDE is very close to the one calculated with KDE.

Figure \ref{fig:mnist-scatter} shows a comparison of the log density predicted by KDE and DMKDE. Both models were applied to test samples and samples generated randomly from a uniform distribution. As expected points are clustered around the diagonal. The DMKDE log density of test samples (left) seems to be more accurately predicted than the one of random samples. The reason is that the density of random samples is smaller than the density of test samples and the difference is amplified by the logarithm. 

Figure \ref{fig:timerawkdevsdmkdesgd} shows the time of both methods for different sizes of the training dataset. The prediction time of KDE depends on the size of the training dataset, while the time of DMKDE does not depend on it. The advantage of DMKDE in terms of computation time is clear for training datasets above $10^4$ data samples.

\begin{figure*}[tbh]
\begin{centering}
\includegraphics[scale=0.36]{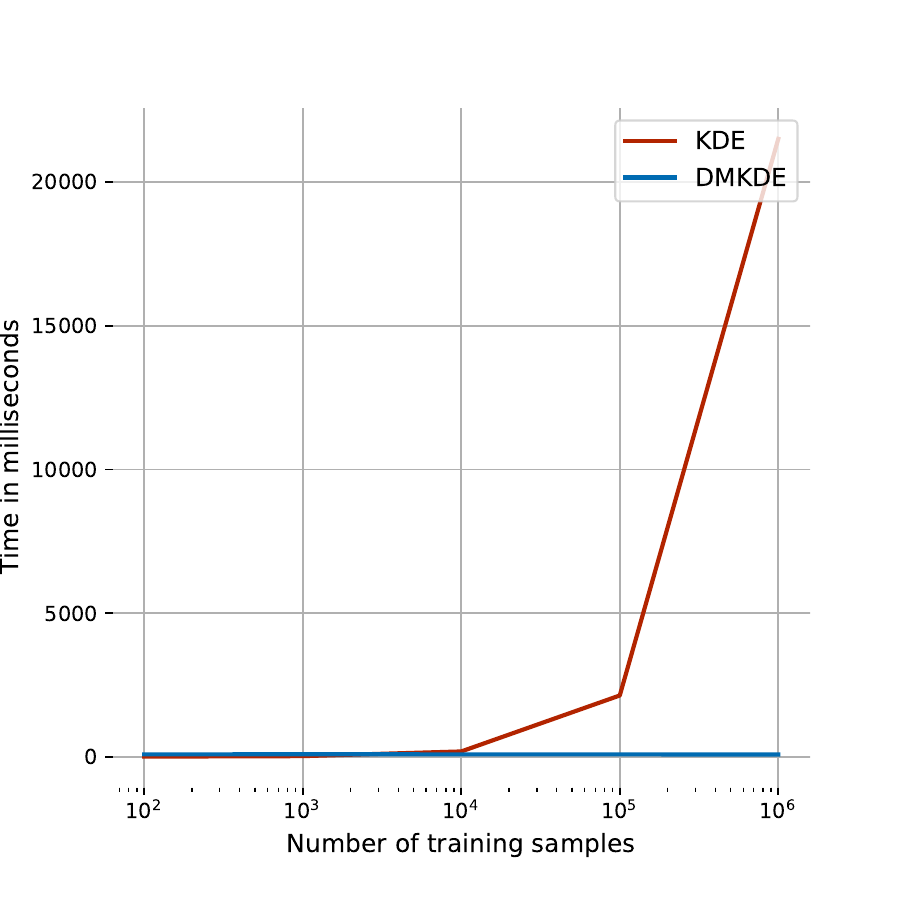}
\includegraphics[scale=0.36]{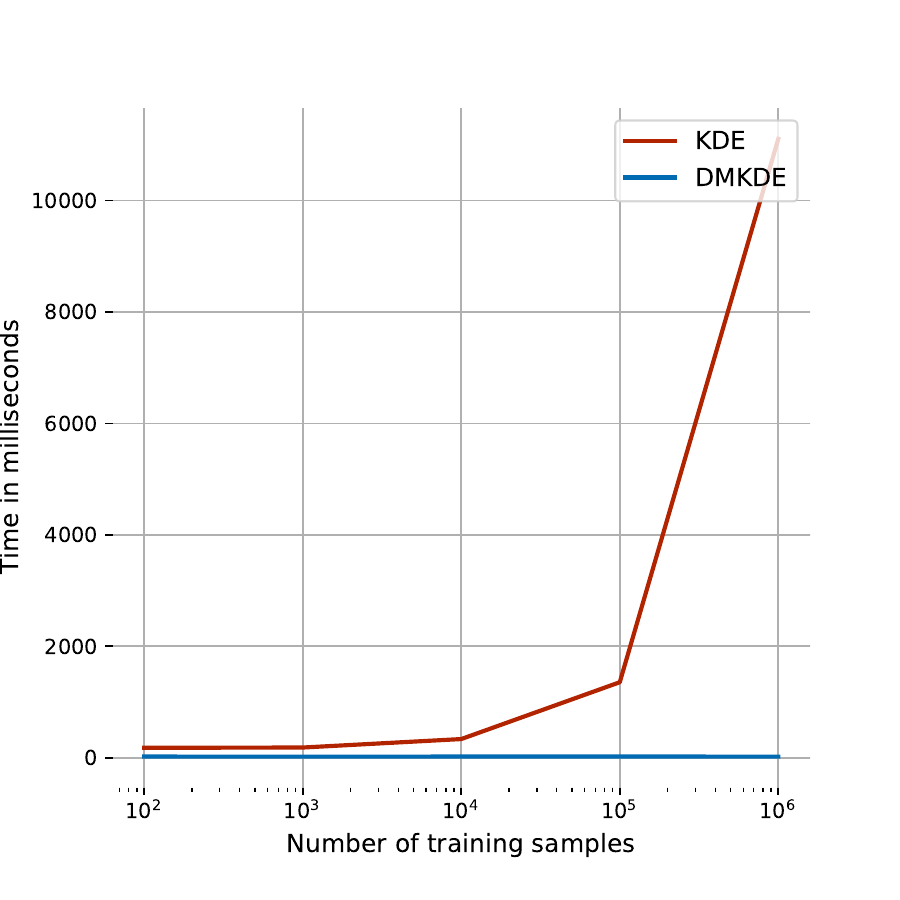}
\includegraphics[scale=0.36]{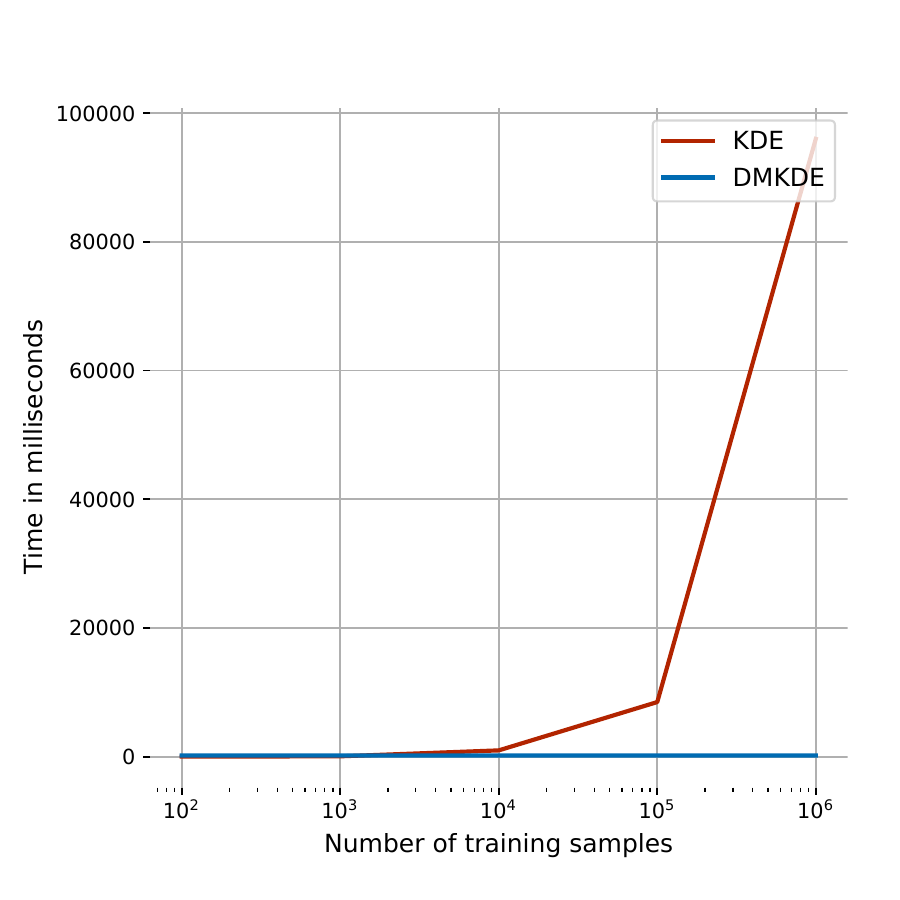}
\par\end{centering}
\caption{Evaluation of the prediction time of DMKDE and KDE: 1-D dataset (top left), 2-D dataset (top right) and MNIST dataset (bottom).
\label{fig:timerawkdevsdmkdesgd}.}
\end{figure*}

\subsection{Classification evaluation}

In this set of experiments, we evaluated DMKDC over different well known benchmark classification datasets. 

\subsubsection{Data sets and experimental setup}

\begin{table}[tbh]
\caption{Data sets used for classification  evaluation.}
\label{Table:classification_dataset}
\vskip 0.15in
\begin{center}
\begin{small}
\begin{sc}
\begin{tabular}{lcccr}
\toprule
Data set & Attributes & Classes & Train-Test \\
\midrule
Letters & 16 & 26 & 14000-6000 \\ 
Usps & 256 & 10 & 7291-2007 \\ 
Forest & 54 & 3 & 70-30 \\ 
Mnist & 784 & 10 & 60000-10000 \\ 
Gisette & 5000 & 2 & 4200-1800 \\ 
Cifar & 3072 & 10 & 60000-10000 \\ 
\bottomrule
\end{tabular}
\end{sc}
\end{small}
\end{center}
\vskip -0.1in
\end{table}

\begin{table}[tbh]
\caption{Specifications of the data sets used for ordinal regression evaluation. Train and Test indicate the number of samples, which is the same for all the twenty partitions.}
\label{datasets_regression}
\vskip 0.15in
\begin{center}
\begin{small}
\begin{sc}
\begin{tabular}{lcccr}
\toprule
Data set & Attributes & Train & Test \\
\midrule
Diabetes & 2 & 30 & 13 \\
Pyrimidines & 27 & 50 & 24 \\
Triazines & 60 & 100 & 86 \\
Wisconsin & 32 & 130 & 64 \\
Machine CPU & 6 & 150 & 59 \\
Auto MPG & 7 & 200 & 192 \\
Boston Housing & 13 & 300 & 206 \\
Stock Domain & 9 & 600 & 350\\
Abalone & 8 & 1000 & 3177 \\
\bottomrule
\end{tabular}
\end{sc}
\end{small}
\end{center}
\vskip -0.1in
\end{table}

Six benchmark data sets were used. The details of these datasets are shown in Table \ref{Table:classification_dataset}. In the case of Gisette and Cifar, we applied a principal component analysis algorithm using 400 principal components in order to reduce the dimension. DMKDC was trained using the estimation strategy (DMKDC) and an ADAM stochastic gradient descent strategy (DMKDC-SGD). As baseline we compared against a linear support vector machine (SVM) trained using the same RFF as DMKDC. The SVM was trained using the LinearSVC model from scikit-learn, which is based in an efficient C implementation tailored to linear SVMs. In the case of MNIST and Cifar, we additionally built a union of a LeNet architecture \citep{lecun1989backpropagation}, as a feature extraction block, and DMKDC-SGD as the classifier layer. The LeNet block is composed of two convolutional layers, one flatten layer and one dense layer. The first convolutional layer has 20 filters, kernel size of 5, same as padding, and ReLu as the activation function. The second convolutional layer has 50 filters, kernel size of 5, same as padding, and ReLu as the activation function. The dense layer has 84 units and ReLU as the activation function. The dense layer is finally connected to DMKDC. We report results for the combined model (LeNet DMKDC-SGD) and the LeNet model with a softmax output layer (LeNet). To make the comparison with baseline models fair, in all the cases the RFF layer of DMKDC-SGD is frozen, so its weights are not modified by the stochastic gradient descent learning process.

For each data set, we made a hyper parameter search using a five-fold cross-validation with 25 randomly generated configurations. The number of RFF was set to 1000 for all the methods. For each dataset we calculated the inter-sample median distance $\mu$ and defined an interval around $\gamma=\frac{1}{2\sigma^2}$. The $C$ parameter of the SVM was explored in an exponential scale from $2^{-5}$ to $2^{10}$. For the ADAM optimizer in DMKDC-SGD with and without LeNet we choose the learning rate in the interval $(0, 0.001]$. The number of eigen-components of the factorization was chosen from $\{0.1,  0.2, 0.5, 1\}$ where each number represents a percentage of the RFF. After finding the best hyper-parameter configuration using cross validation, 10 different experiments were performed with different random initialization. The mean and the standard deviation of the accuracy is reported.

\subsubsection{Results and discussion}

\begin{sidewaystable}
\sidewaystablefn
\caption{Accuracy test results for DMKDC and DMKDC-SGD compared against a linear support vector machine over RFF (SVM-RFF). Two deep learning models are also evaluated on the two image datasets: a convolutional neural network (LeNet) and its combination with DMKDC-SGD (LeNet DMKDC).}
\label{QMC_results}
\vskip 0.15in
\begin{center}
\begin{small}
\begin{sc}
\begin{tabular}{lccccccc}
\toprule
Data Set & SVM-RFF & DMKDC & DMKDC-SGD & LeNet & LeNet DMKDC  \\ 
\midrule
Letters & 0.924$\pm$0.002 & 0.918$\pm$0.002 & \textbf{0.944$\pm$0.002} & - & - \\
USPS & 0.940$\pm$0.001 & 0.806$\pm$0.004 & \textbf{0.967$\pm$0.002} & - & - \\ 
Forest & 0.697$\pm$0.046 & 0.727$\pm$0.002 & \textbf{0.876$\pm$0.001} & - & - \\ 
Gisette  & 0.944$\pm$0.003 & 0.836$\pm$0.003 & \textbf{0.957$\pm$0.001} & - & - \\ 
Mnist & 0.950$\pm$0.006 & 0.811$\pm$0.001 & 0.952$\pm$0.004 & 0.989$\pm$0.001 & \textbf{0.989$\pm$0.001} \\ 
Cifar & 0.453$\pm$0.003 & 0.271$\pm$0.005 & 0.484$\pm$0.006 & 0.616$\pm$0.001 & \textbf{0.628$\pm$0.001} \\ 
\bottomrule
\end{tabular}
\end{sc}
\end{small}
\end{center}
\vskip -0.1in
\end{sidewaystable}

Table \ref{QMC_results} shows the results of the classification experiments. DMKDC is a shallow method that uses RFF, so a SVM using the same RFF is fair and strong baseline. In all the cases, except one, DMKDC-SGD outperforms the SVM, which shows that it is a very competitive shallow classification method. DMKDC trained using estimation shows less competitive results, but they are still remarkable taking into account that this is an optimization-less training strategy that only passes once over the training dataset.  For MNIST and Cifar the use of a deep learning feature extractor is mandatory to obtain competitive results. The results show that DMKDC-SGD can be integrated with deep neural network architectures to obtain competitive results.

The improvement on classification performance of DMKC-SGD comes at the cost of increased training time. The training of DMKDC is very efficient since it corresponds to do an average of the training density matrices. Linear SVM training is also very efficient. In contrast, DMKDC-SGD requires an iterative training process that has to be tuned to get it to converge to a good local optimum, as is the case for current deep learning models. 

\subsection{Ordinal regression evaluation}
Many multi-class classification problems can be seen as ordinal regression problems. That is, problems where labels not only indicate class membership, but also an order. Ordinal regression problems are halfway between a classification problem and a regression problem, and given the discrete probability distribution representation used in QMR, ordinal regression seems to be a suitable problem to test it.

%MAE is a measure that optimizes the total number of errors, although it can be tolerant of large errors.

\subsubsection{Data sets and experimental setup}

\begin{sidewaystable}
\caption{MAE test results for QMR, QMR-SGD and different baseline methods: support vector machines (SVM), Gaussian Processes (GP), Neural Network Rank (NNRank), Ordinal Extreme Learning Machines (ORELM) and Ordinal Regression Neural Network (ORNN). The results are the mean and standard deviation of the MAE for the twenty partitions of each dataset. The best result for each data set is in bold face.}
\label{QMOR_results}
\vskip 0.15in
\begin{center}
\begin{small}
\begin{sc}
\begin{tabular}{lccccccc}
\toprule
Data set & QMR-SGD & QMR & ORNN &  NNRank & GP & SVM \\
\midrule
Diabetes & \textbf{0.511$\pm$0.07} & 0.611$\pm$0.02& -- &  0.546$\pm$0.15 & 0.662$\pm$0.14 & 0.746$\pm$0.14 \\
Pyrimidines & 0.408$\pm$0.07 & 0.946$\pm$0.06 & 0.677$\pm$0.20 &  0.450$\pm$0.10 & \textbf{0.392$\pm$0.07} & 0.450$\pm$0.11 \\
Triazines & \textbf{0.674$\pm$0.02} & 0.695$\pm$0.01 & -- & 0.730$\pm$0.07 & 0.687$\pm$0.02 & 0.698$\pm$0.03 \\
Wisconsin & \textbf{0.985$\pm$0.04}& 1.114$\pm$0.04 & -- &  -- & 1.010$\pm$0.09 & 1.003$\pm$0.07 \\
Machine & \textbf{0.171$\pm$0.03} & 0.995$\pm$0.07 & 0.451$\pm$0.03  & 0.186$\pm$0.04 & 0.185$\pm$0.04  & 0.192$\pm$0.04 \\
Auto & \textbf{0.230$\pm$0.02} & 0.710$\pm$0.02 & -- &  0.281$\pm$0.02 & 0.241$\pm$0.02 & 0.260$\pm$0.02 \\
Boston & 0.270$\pm$0.02 & 0.679$\pm$0.02 & --  & 0.295$\pm$0.04 & \textbf{0.259$\pm$0.02} & 0.267$\pm$0.02 \\
Stocks & \textbf{0.103$\pm$0.01} & 0.971$\pm$0.00 & 0.127$\pm$0.01  & 0.127$\pm$0.02 & 0.120$\pm$0.02 & 0.108$\pm$0.02 \\
Abalone & 0.233$\pm$0.01 &0.307$\pm$0.00 & 0.635$\pm$0.01  & \textbf{0.226$\pm$0.01} & 0.232$\pm$0.00 & 0.229$\pm$0.00 \\
\bottomrule
\end{tabular}
\end{sc}
\end{small}
\end{center}
\vskip -0.1in
\end{sidewaystable}

Nine standard benchmark data sets for ordinal regression were used. The details of each data set are reported in Table~\ref{datasets_regression}. These data sets are originally used in metric regression tasks. To convert the task into an ordinal regression one, the target values were discretized by taking five intervals of equal length over the target range. For each set, 20 different train and test partitions are made. These partitions are the same used by \citet{Chu2005} and several posterior works, and are publicly available at
\url{http://www.gatsby.ucl.ac.uk/~chuwei/ordinalregression.html}. The models were evaluated using the mean absolute error (MAE), which is  a popular and widely used measure in ordinal regression  \citep{Gutierrez2016,Garg2020}.

QMR was trained using the estimation strategy (QMR) and an ADAM stochastic gradient descent strategy (QMR-SGD). For each data set, and for each one of the 20 partitions, we made a hyper parameter search using a five-fold cross-validation procedure. The search was done generating 25 different random configuration. The range for $\gamma$ was computed in the same way as for the classification experiments, $\beta \in (0, 25)$, the number of RFF randomly chosen between the number of attributes and $1024$, and the number of eigen-components of the factorization was chosen from $\{0.1,  0.2, 0.5, 1\}$ where each number represents a percentage of the RFF. For the ADAM optimizer in QMR-SGD we choose the learning rate in the interval $(0, 0.001]$ and $\alpha \in (0,1)$. The RFF layer was always set to trainable, and the criteria for selecting the best parameter configuration was the MAE performance.

\subsubsection{Results and discussion}

 For each data set, the means and standard deviations of the test MAE for the 20 partitions are reported in Table \ref{QMOR_results}, together with the results of previous state-of-the-art works on ordinal regression: Gaussian Processes (GP) and support vector machines (SVM) \citep{Chu2005}, Neural Network Rank (NNRank) \citep{Cheng2008}, Ordinal Extreme Learning Machines (ORELM) \citep{Deng2010} and Ordinal Regression Neural Network (ORNN) \citep{Fernandez-Navarro2014}.

QMR-SGD shows a very competitive performance. It outperforms the baseline methods in six out of the nine data sets. The training strategy based on estimation, QMR, did not performed as well. This evidences that for this problem a fine tuning of the representation is required and it is successfully accomplished by the gradient descent optimization. 

\section{Conclusions}
\label{sec:conclusions}

The mathematical framework underlying quantum mechanics is a powerful formalism that harmoniously combine linear algebra and probability in the form of density matrices. This paper has shown how to use these density matrices as a building block for designing different machine learning models. The main contribution of this work is to show a novel perspective to learning that combines two very different and seemingly unrelated tools, random features and density matrices. The, somehow surprising, connection of this combination with kernel density estimation provides a new way of representing and learning probability density functions from data. The experimental results showed some evidence that this building block can be used to build competitive models for some particular tasks. However, the full potential of this new perspective is still to be explored. Examples of directions of future inquire include using complex valued density matrices, exploring the role of entanglement and exploiting the battery of practical and theoretical tools provided by quantum information theory.

\section*{Statements and Declarations}

The authors declare that they have no known competing interests.

\bibliography{learning_with_dm_and_rf}

\newpage

\begin{appendices}
\section{Proofs}\label{ap:proofs}
\begin{proposition}
    Let $\mathcal{M}$ be a compact subset of $\mathbb{R}^d$ with a diameter $\text{diam}(\mathcal{M})$, let $X=\{x_i\}_{i=1\dots N}\subset \mathcal{M}$ a set of iid samples,  then  $\hat{f}_{\rho_{\mathrm{train}}}$ (\cref{eq:f_hat_rho_train}) and $\hat{f}_{\gamma}$  satisfy:
    
      \begin{align}
      \mathrm{Pr} &\left[
      \sup_{x \in \mathcal{M}} 
      \vert \hat{f}_{\rho_{\mathrm{train}}}(x) - \hat{f}_{\gamma}(x)\vert  
      \ge \epsilon 
      \right] \le \nonumber\\
      &2^8\left(\frac{\sqrt{2d\gamma} \mathrm{diam}(\mathcal{M})}{3M_{\gamma}\epsilon}\right)^2
      \exp\left(-\frac{D(3M_{\gamma}\epsilon)^2}{4(d+2)}\right)
      \end{align}
\end{proposition}

\begin{proof}
  \begin{align}\label{eq:f_rho_sum_phi_rff}
    \hat{f}_{\rho_{\mathrm{train}}}(x) & =  \frac{1}{M_{\gamma}}\phi_{\text{rff}}^T(x) \rho_{\mathrm{train}} \phi_{\text{rff}}(x)  \nonumber\\
    & = \frac{1}{M_{\gamma}}\phi_{\text{rff}}^T(x) \left( \frac{1}{N}\sum_{i=1}^N \phi_{\text{rff}}(x_i) \phi_{\text{rff}}^T(x_i)\right) \phi_{\text{rff}}(x) \nonumber\\
    & = \frac{1}{M_{\gamma}N}\sum_{i=1}^N \phi_{\text{rff}}^T(x) \phi_{\text{rff}}(x_i) \phi_{\text{rff}}^T(x_i) \phi_{\text{rff}}(x) \nonumber\\
    & = \frac{1}{M_{\gamma}N}\sum_{i=1}^N (\phi_{\text{rff}}^T(x) \phi_{\text{rff}}(x_i))^2
  \end{align}
  Remembering that in \cref{eq:f_hat_rho_train} we used a spread parameter of $\frac{\gamma}{2}$ to calculate the parameters of $\phi_{\text{rff}}$ and   because of Theorem \ref{thm:rff} we know that 
  \begin{multline*}
      \mathrm{Pr} \left[
      \sup_{x,y \in \mathcal{M}} \vert \phi_{\mathrm{rff}}^T(x)\phi_{\mathrm{rff}}(y) - e^{-\frac{\gamma}{2}\|x-y\|^2} \vert 
      \ge \epsilon \right] \le \\
      2^8\left(\frac{\sqrt{d\gamma} \mathrm{diam}(\mathcal{M})}{\epsilon}\right)^2
      \exp\left(-\frac{D\epsilon^2}{4(d+2)}\right) = B
  \end{multline*}
  By construction $\vert \phi_{\mathrm{rff}}^T(x)\phi_{\mathrm{rff}}(y) + e^{-\frac{\gamma}{2}\|x-y\|^2}\vert  \le 3$, then
  $\vert (\phi_{\mathrm{rff}}^T(x)\phi_{\mathrm{rff}}(y))^2 - e^{-\gamma\|x-y\|^2}\vert 
  = \vert (\phi_{\mathrm{rff}}^T(x)\phi_{\mathrm{rff}}(y) - e^{-\frac{\gamma}{2}\|x-y\|^2})(\phi_{\mathrm{rff}}^T(x)\phi_{\mathrm{rff}}(y) + e^{-\frac{\gamma}{2}\|x-y\|^2})\vert 
  \le 3\vert (\phi_{\mathrm{rff}}^T(x)\phi_{\mathrm{rff}}(y) - e^{-\frac{\gamma}{2}\|x-y\|^2})\vert  $.
  Then  
    \begin{equation}\label{eq:diff_squares_bound}
      \mathrm{Pr} \left[
      \sup_{x,y \in \mathcal{M}} \vert (\phi_{\mathrm{rff}}^T(x)\phi_{\mathrm{rff}}(y))^2 - e^{-\gamma\|x-y\|^2} \vert 
      \ge 3\epsilon \right] \le B
  \end{equation}
  
 Combining Equations \cref{eq:f_rho_sum_phi_rff} and \cref{eq:diff_squares_bound} we get:
\begin{equation*}
      \mathrm{Pr} \left[
      \sup_{x \in \mathcal{M}} 
      \vert \hat{f}_{\rho}(x) - \hat{f}_{\gamma}(x)\vert  
      \ge 3M_{\gamma}\epsilon 
      \right] \le B
\end{equation*}

Making a variable change we get:

      \begin{align}
      \mathrm{Pr} &\left[
      \sup_{x \in \mathcal{M}} 
      \vert \hat{f}_{\rho_{\mathrm{train}}}(x) - \hat{f}_{\gamma}(x)\vert  
      \ge \epsilon 
      \right] \le \nonumber\\
      &2^8\left(\frac{\sqrt{2d\gamma} \mathrm{diam}(\mathcal{M})}{3M_{\gamma}\epsilon}\right)^2
      \exp\left(-\frac{D(3M_{\gamma}\epsilon)^2}{4(d+2)}\right)
      \end{align}
\end{proof}




\end{appendices}

\end{document}